\documentclass{article}

\newif\ifarxivpreprint
\arxivpreprintfalse

\ifarxivpreprint
    \usepackage{arxiv}
    \usepackage{natbib}
\else
        \usepackage[final]{neurips_2025}

\fi

\usepackage[utf8]{inputenc} % allow utf-8 input
\usepackage[T1]{fontenc}    % use 8-bit T1 fonts
\usepackage{hyperref}       % hyperlinks
\usepackage{url}            % simple URL typesetting
\usepackage{booktabs}       % professional-quality tables
\usepackage{amsfonts}       % blackboard math symbols
\usepackage{nicefrac}       % compact symbols for 1/2, etc.
\usepackage{microtype}      % microtypography
\usepackage{xcolor}         % colors

\usepackage{amsfonts,amsmath,amssymb,amsthm}
\usepackage{mathtools}
\usepackage{algorithm,algpseudocode}
\usepackage{graphicx} % Required for inserting images
\usepackage{multirow}
\usepackage{siunitx} % For proper number alignment
\usepackage{array}
\usepackage{makecell}
\usepackage{geometry}
\usepackage{dsfont}
\usepackage{pifont}
\usepackage{subcaption}
\usepackage[capitalise]{cleveref}
\usepackage{defs}

\hypersetup{
	colorlinks,
	linkcolor={blue!50!gray},
	citecolor={blue!50!gray},
	urlcolor={blue!50!gray}
}

\title{Turbocharging Gaussian Process Inference with Approximate Sketch-and-Project}

\author{%
  Pratik Rathore \\
  \stanford \\
  \texttt{pratikr@stanford.edu} \\
  \And
  Zachary Frangella \\
  \stanford \\
  \texttt{zfran@stanford.edu} \\
  \And
  Sachin Garg \\
  \umich \\
  \texttt{sachg@umich.edu} \\
  \And
  Shaghayegh Fazliani \\
  \stanford \\
  \texttt{fazliani@stanford.edu} \\
  \And
  \michal{} \derezinski{} \\
  \umich \\
  \texttt{derezin@umich.edu} \\
  \And
  Madeleine Udell \\
  \stanford \\
  \texttt{udell@stanford.edu} \\
}

\begin{document}

\maketitle

\begin{abstract}
Gaussian processes (GPs) play an essential role in biostatistics, scientific machine learning, and Bayesian optimization for their ability to provide probabilistic predictions and model uncertainty.
However, GP inference struggles to scale to large datasets (which are common in modern applications), since it requires the solution of a linear system whose size scales quadratically with the number of samples in the dataset.
We propose an approximate, distributed, accelerated sketch-and-project algorithm (\adasap{}) for solving these linear systems, which improves scalability.
We use the theory of determinantal point processes to show that the posterior mean induced by sketch-and-project rapidly converges to the true posterior mean. 
In particular, this yields the first efficient, condition number-free algorithm for estimating the posterior mean along the top spectral basis functions, showing that our approach is principled for GP inference.
\adasap{} outperforms state-of-the-art solvers based on conjugate gradient and coordinate descent across several benchmark datasets and a large-scale Bayesian optimization task.
Moreover, \adasap{} scales to a dataset with $> 3 \cdot 10^8$ samples, a feat which has not been accomplished in the literature.

\end{abstract}

\section{Introduction}
Gaussian processes (GPs) are a mainstay of modern machine learning and scientific computing, due to their ability to provide probabilistic predictions and handle uncertainty quantification.
% , which is critical in many scientific and engineering applications.
Indeed, GPs arise in applications spanning Bayesian optimization \citep{hernandezlobato2017parallel}, genetics \citep{mcdowell2018clustering}, health care analytics \citep{cheng2020sparse}, materials science \citep{frazier2016bayesian}, and partial differential equations \citep{chen2025sparse}. 

GP inference for a dataset with $n$ samples requires the solution of a dense $n \times n$ linear system, which is challenging to solve at large scale.
Direct methods like Cholesky decomposition require $\bigO(n^3)$ computation, limiting them to $n \sim 10^4$.
Consequently, two main approaches have arisen for large-scale inference: 
(i) exact inference based on iterative methods for solving linear systems \citep{wang2019exact,lin2023sampling,lin2024stochastic}
and (ii) approximate inference based on inducing points and variational methods \citep{titsias2009variational,hensman2013gaussian}.
The state-of-the-art approaches for large-scale inference are preconditioned conjugate gradient (\pcg{}) \citep{wang2019exact} and stochastic dual descent (SDD) \citep{lin2024stochastic}. 
\pcg{} offers strong convergence guarantees and good performance on ill-conditioned problems, but is slow when $n \sim 10^{6}$.
In contrast, SDD scales to $n \gg 10^6$, but lacks strong theoretical guarantees, can slow down in the face of ill-conditioning, and introduces a stepsize parameter which can be challenging to tune.
Inducing points and variational methods \citep{titsias2009variational,hensman2013gaussian} scale to large datasets, but are often outperformed by exact inference \citep{wang2019exact,lin2023sampling,lin2024stochastic}.
Altogether, the current state of algorithms for large-scale GP inference is unsatisfactory, as practitioners have to trade-off between quality of inference, robustness to ill-conditioning, ease of setting hyperparameters, and scalability.

To address this gap in the literature, we introduce the \textbf{A}pproximate \textbf{D}istributed \textbf{A}ccelerated \textbf{S}ketch-\textbf{a}nd-\textbf{P}roject (\adasap{}) algorithm.
\adasap{} is rooted in the sketch-and-project framework of \cite{gower2015randomized}, and obtains the robustness to ill-conditioning of \pcg{} and the scalability of \sdd{}, while 
having reliable, default hyperparameters that work well in practice.
For example, \adasap{} dramatically outperforms tuned \sdd{} and \pcg{} on a dataset with $n > 10^6$ samples (\cref{fig:intro_houseelec}). 

Our contributions are as follows:
\begin{enumerate}
    \item We develop \adasap{} for large-scale GP inference. 
    \adasap{} uses approximate preconditioning and acceleration to address ill-conditioning, and uses distributed computing to improve the speed of bottleneck operations.
    \adasap{} also comes with effective default hyperparameters that work out of the box. 
    \item Using the theory of determinantal point processes, we show that \sapfull-style methods converge faster than stochastic first-order methods like \sdd{} in the presence of ill-conditioning. In particular, we give a first-of-its-kind condition number-free time complexity bound for estimating the GP posterior mean to moderate accuracy, explaining the excellent test error performance of \adasap{}.
    \item We empirically verify that \adasap{}, with its default hyperparameters, outperforms state-of-the-art competitors on benchmark large-scale GP inference tasks, and is capable of scaling to a dataset with $n > 3 \cdot 10^8$ samples.
    \item We show \adasap{} yields the best performance on a large-scale Bayesian optimization task from \cite{lin2023sampling}.
\end{enumerate}

% \begin{itemize}
%     \item Gaussian processes \citep{rasmussen2005gaussian} are used in a variety of applications, such as biostatistics, scientific ML, and Bayesian optimization (add more examples)
%     \item Talk about how GP inference is used, and the associated scalability issues
%     \item Briefly discuss attempts at scaling, e.g., inducing points and \pcg{}
%     \item Bring up approximate \sapfull{} as a scalable alternative
%     \item Make it clear that tail averaging is of theoretical interest; we run \adasap{} without tail averaging in practice
%     \item Go into contributions of the paper
%     \item Give a plot illustrating the effectiveness of our method
%     \item We can draw from \citep{lin2023sampling,lin2024stochastic} for ideas on how to write the intro
%     \item Add paragraph for notation
% \end{itemize}

% \paragraph{Notation}
% Given a symmetric psd matrix $A \in \R^{n\times n}$, we denotes its eigenvalues in descending order: $\lambda_1(A)\geq\lambda_2(A)\ldots\geq\lambda_n(A)$.
%\mnote{I don't think we need this here.}

\begin{figure}
    \centering
    \includegraphics[width=0.9\linewidth]{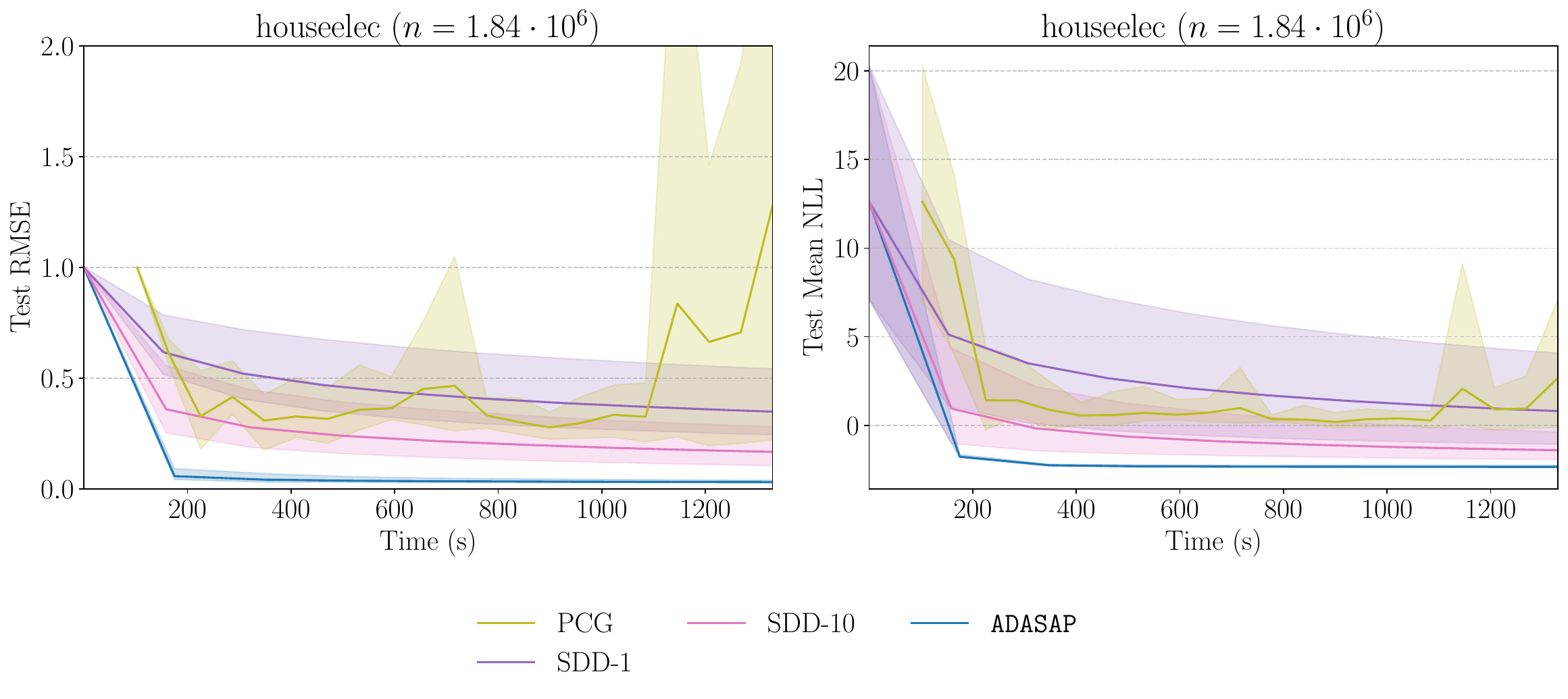}
    \caption{\adasap{} attains lower root mean square error (RMSE) and mean negative log likelihood (NLL) than start-of-the-art methods \sdd{} \citep{lin2024stochastic} and PCG on the houseelec dataset.
    \sdd-1 and \sdd-10 correspond to two particular stepsize selections for the \sdd{} method.
    The solid lines indicate the mean performance of each method, while the shaded regions indicate the range between worst and best performance of each method over five random splits of the data.
    }
    % \mnote{I would reduce the error bars somehow, they look quite jarring. Also, it looks slightly odd to show two plots of the same dataset here. At least, the name should not appear twice at the top.}}
    \label{fig:intro_houseelec}
\end{figure}
\section{GP Regression and Inference}
\label{sec:gp_fundamentals}
Let $\Xc$ be a set. 
A random function $f \colon \Xc \rightarrow \R$ is called a \textit{Gaussian process} if for any finite subset $\{x_1, x_2, \ldots, x_n\} \subset \Xc$ of points (where $n$ is any positive integer), the random vector $(f(x_1), f(x_2), \ldots, f(x_n))$ follows a multivariate Gaussian distribution \citep{rasmussen2005gaussian}.
A Gaussian process is typically denoted as $f \sim \GPc(m, k)$, where $m \colon \Xc \rightarrow \mathbb{R}$ is the mean function and $k \colon \Xc \times \Xc \rightarrow \mathbb{R}$ is the covariance function (or kernel). 
Throughout this paper, the kernel function $k$ follows broadcasting conventions, where operations between individual points and sets of points produce vectors or matrices of kernel evaluations as appropriate.

Suppose we have a Gaussian process prior $f \sim \GPc(m, k)$.
Given observations $(X, y) \in \R^{n \times d} \times \R^n$ and likelihood variance $\lambda > 0$, we would like to perform inference using the GP posterior, i.e., (i) sampling from the posterior and (ii) computing the posterior mean. 
% \mnote{The variance $\lambda$ is never explicitly used to specify the model of $y$.}

For conciseness, define $K \coloneqq k(X, X) \in \R^{n \times n}$.
Then a sample from the GP posterior is a random function $f_n \sim \GPc(m_n, k_n)$ \citep{rasmussen2005gaussian}, where
\begin{align}
\label{eq:post_eqns}
    % & K = k(X, X), k(x, X) \coloneqq [k(x, x_1), k(x, x_2), \ldots, k(x, x_n)], k(X, x') \coloneqq [k(x_1, x'), k(x_2, x'), \ldots, k(x_n, x')]^T\\
    & m_n(\cdot) = m(\cdot) + k(\cdot, X)(K + \lambda I)^{-1} y, \\
    & k_n(\cdot, \cdot) = k(\cdot,\cdot) - k(\cdot, X)(K + \lambda I)^{-1}k(X, \cdot).
\end{align}

The standard approach for drawing posterior samples at locations $X_*$ is to perform a linear transformation of standard Gaussian random variables $\zeta \sim \Nc(0, I)$ \citep{wilson2020efficiently,wilson2021pathwise}:
\begin{equation}
    \label{eq:standard_posterior_sample}
    f_n(\Xinf) = m_n(\Xinf) + k_n(\Xinf, \Xinf)^{1/2} \zeta.  
\end{equation}
Unfortunately, the approach in \eqref{eq:standard_posterior_sample} does not scale when $|\Xinf|$ is large, since the matrix square root $k_n(\Xinf, \Xinf)^{1/2}$ requires $O(|\Xinf|^3)$ computation.

\subsection{Pathwise conditioning}
\label{subsec:path_cond}
To address the scaling challenges in \eqref{eq:standard_posterior_sample}, \cite{wilson2020efficiently, wilson2021pathwise} develop pathwise conditioning.
Pathwise conditioning rewrites \eqref{eq:standard_posterior_sample} as
\begin{equation}
   \label{eq:pathwise_conditoning} 
   f_n(\Xinf) = f(\Xinf) + k_n(\Xinf, X)(K+\lambda I)^{-1}(y - f(X_n) - \zeta), 
   \quad \zeta
   \sim \Nc(0,\lambda I), \quad f \sim \GPc(0,k). 
\end{equation}
Equation \eqref{eq:pathwise_conditoning} enables scalable sampling from the posterior as $f$ can be approximated using random features \citep{rahimi2007random, rudi2017generalization} for $k$, which costs $\bigO(q|\Xinf|)$, where $q$ is the number of random features. 
This yields a significant reduction in cost from the standard posterior sampling scheme in \eqref{eq:standard_posterior_sample}: by leveraging pathwise conditioning, posterior sampling can be performed by generating $f$ via random features \citep{wilson2020efficiently,wilson2021pathwise} and solving $(K+\lambda I)^{-1}(y - f(X_n) - \zeta)$ using an iterative method like \pcg{} or \sdd{}.
Our experiments (\cref{sec:experiments}) use pathwise conditioning to sample from the posterior.
% The cost of solving the linear system in \eqref{eq:pathwise_conditoning} with \pcg{} is $\bigOt(n^2)$, which is a significant improvement over $\bigO(|\Xinf|^3)$.

% While pathwise conditioning reduces the cost of generating posterior samples, it still requires solving a large linear system with multiple right-hand sides, one for each test point.
% When $n$ is very large, these large linear systems can only be solved with iterative methods.
% Popular methods for solving this linear system include \pcg{} \citep{wang2019exact}, SGD \citep{lin2023sampling}, and more recently Stochastic Dual Descent (\sdd{}) \citep{lin2024stochastic}.
% The recent works \cite{lin2023sampling,lin2024stochastic} have shown stochastic methods like SGD and \sdd{} outperform \pcg{} and tend to be far more scalable, as they avoid a $\bigO(n^2)$ per-iteration cost.
% In addition to better scalability, these methods can better handle single (or lower) precision than \pcg{} \citep{wu2024large, rathore2025have}, which is essential to maximizing GPU performance and reducing storage costs \citep{abdelfattah2021survey}.    

\subsection{Drawbacks of state-of-the-art exact inference methods}
The current state-of-the-art for large-scale posterior sampling is \sdd{} \citep{lin2024stochastic}, a stochastic first-order algorithm combining coordinate descent \citep{richtarik2014iteration,qu2016coordinate} with momentum and geometric averaging.
Stochastic methods like \sdd{} can better handle single (or lower) precision than \pcg{} \citep{wu2024large, rathore2025have}, which is essential to maximizing GPU performance and reducing storage costs \citep{abdelfattah2021survey}.
While \sdd{} overcomes the scalability challenges associated with \pcg{}, it faces challenges due to the ill-conditioning of $K + \lambda I$, which determines its worst-case convergence rate \citep{nemirovskij1983problem}.
% This limitation is inherent to stochastic first-order optimization algorithms, as their worst-case rate is determined by the condition number of the problem \citep{nemirovskij1983problem}.
% For GP inference, the condition number of the problem is determined by the condition number of $K +\lambda I$.
% As $K + \lambda I$ is often ill-conditioned, first-order methods like SGD and \sdd{} converge slowly.
% the convergence of first-order methods like SGD and \sdd{} becomes very slow. 
% Therefore, for reliable large-scale inference, it is critical to develop a method capable of solving large-scale linear systems with multiple right-hand sides in parallel, while also addressing the ill-conditioning of 
% $K + \lambda I$. 
Therefore, reliable large-scale inference requires a method that has cheaper per-iteration costs than \pcg{} and can handle the ill-conditioning of $K + \lambda I$.
However, to the best of our knowledge, no existing method effectively tackles both issues: SGD and \sdd{} offer speed and scalability but suffer from ill-conditioning, while \pcg{} offers robustness to ill-conditioning but struggles to scale beyond $n > 10^6$.
In \cref{sec:sap} we propose \sap{}, which is robust to ill-conditioning; in \cref{sec:adasap}, we develop \adasap{}, an extension of \sap{} that scales to large datasets.

% \znote{Add ref later to figure showing how first-order methods make very slow progress}
\section{\Sapfull{} for GP Inference}
\label{sec:sap}
% \mnote{Moved the theory here, as discussed in the meeting.}
% As we have seen in \cref{subsec:path_cond}, the need to solve a large ill-conditioned linear system involving $K + \lambda I$ efficiently
% is the main challenge in large-scale GP inference.
\cref{sec:gp_fundamentals} shows the main challenge of large-scale GP inference is the need to solve a large linear system with the ill-conditioned matrix $K+\lambda I$, which cannot be done satisfactorily with \pcg{} or \sdd{}. 
% Methods like PCG handle ill-conditioning well, but scale poorly.
% Conversely, stochastic first-order methods like SGD and SDD scale to large $n$, but struggle with ill-conditioning.
% This paper develops \adasap{}, a method 
% that is both scalable and robust to ill-conditioning.
% \adasap{} builds off of the sketch-and-project  framework (\sap{}, introduced by \cite{gower2015randomized}), 
% and enjoys the benefit of being far more scalable, as well as robust to ill-conditioning.
We introduce the \sap{} algorithm, which effectively addresses the conditioning challenges of GP inference.
However, \sap{} alone does not address the scalability challenges of GP inference, motivating the development of \adasap{} in \cref{sec:adasap}.

\subsection{Sketch-and-project}
\begin{algorithm}[t]
\caption{\sap{} for $\Klambd W = Y$}
    \label{alg:sap}
    % \scriptsize
    \begin{algorithmic}
        \Require blocksize $b$, distribution of row indices $\D$ over $[n]^{b}$, number of iterations $T_{\max}$, initialization $\iterw{0}{} \in \R^{n \times \nrhs}$, averaging boolean tail\_average
        \State
        \For{$t = 0, 1, \ldots, T_{\max} - 1$}
        \State Sample row indices $\B$ of size $b$ from $[n]$ according to $\D$
        \State $\iterd{t}{} \gets \Ib^T (\Kbb + \lambda I)^{-1} (\Kbn \iterw{t}{} + \lambda \Ib \iterw{t}{} - Y_\B)$ \Comment{Search direction; costs $\bigO(bn\nrhs + b^3)$ }
        \State $\iterw{t+1}{} \gets \iterw{t}{} - \iterd{t}{}$ \Comment{Update parameters; costs $\bigO(b\nrhs)$}
        % \Comment{$\iterw{t+1}{} = \arg\min_w \|w - \iterw{t}{}\|_{\Klambd}^2 \text{ s.t. } \iterS{t}{} \Klambd w = \iterS{t}{} y$}
        \If{tail\_average}
        \State $\bar W_{t + 1} \gets 2 / (t + 1) \sum_{j = (t + 1)/2}^{t} W_j$
        \EndIf
        \EndFor
        \State \Return $\bar W_T$ if tail\_average else $\iterw{T}{}$ \Comment{Approximate solution to $\Klambd W = Y$}
    \end{algorithmic}
\end{algorithm}
We formally present \sap{} in \cref{alg:sap}.
\sap{} resembles randomized block coordinate descent, except that it employs \textit{subspace preconditioning}. 
That is, it preconditions the subspace gradient $\Kbn W_t+\lambda \Ib W_t - Y_\B$ by the inverse of the subspace Hessian $\Kbb + \lambda I$.
Thus, unlike SDD, \sap{} includes second-order information in its search direction. 
In addition, we have also augmented vanilla \sap{} from \cite{gower2015randomized} with tail averaging \citep{jain2018parallelizing, epperly2024randomized},
which helps address the inherent noise in the iterates, and plays a crucial role in our analysis.

\sap{} (assuming the number of right-hand sides $\nrhs = 1$ for simplicity) only costs $\bigO(bn + b^3)$ per-iteration, which is significantly smaller than the $\bigO(n^2)$ iteration cost of \pcg, while being comparable with \sdd, which costs $\bigO(bn)$ per-iteration.
Furthermore, \pcg{} typically requires memory linear in $n$ to store the preconditioner, while \sap{} only incurs and additional $\bigO(b^2)$ storage cost.
% \sap{} also compares favorably with algorithms like SGD and \sdd, which cost $\bigO(bn)$ per-iteration. \mnote{What do you mean by compares "favorably"? $O(bn)$ is less than $O(bn+b^3)$}
Thus, \sap{} strikes a balance between \pcg{} and methods like \sdd. 
It incorporates second-order information, like \pcg{}, but does so at a modest increase in cost relative to \sdd{} and SGD.
% Moreover, in \cref{sec:adasap}, we introduce a new variant of \sap{}, whose per-iteration cost is even smaller compared to \sap{}, bringing it closer to SGD and \sdd.
% If we can show that \sap{}'s incorporation of subspace preconditioning and tail averaging improves the dependence upon the condition number, then this would establish \sap{} as the solution to the conditioning challenges discussed at the beginning of this section.
By leveraging the theory of determinantal point processes \citep{kulesza2012determinantal}, we show that tail-averaged \sap{} improves dependence upon the condition number in the early phase of the convergence, which establishes it as a potential solution to the conditioning challenges discussed in \cref{sec:gp_fundamentals}, particularly in the presence noisy data and generalization error.
% We do this by leveraging the theory of determinantal point processes \citep{kulesza2012determinantal}.

\subsection{Fast convergence along \topl{} subspace for GP inference}
\label{subsec:fst_subspace_conv}
% As we have seen, being able to solve linear systems with $K_\lambda$ efficiently is essential to GP inference via pathwise conditioning, 
% and so, our solver should possess strong convergence guarantees.
% We argue that \sap{} with tail averaging enjoys faster convergence than SGD and \sdd{}.
Prior works \citep{dicker2017kernel,lin2020optimal,lin2023sampling} have shown that the components of the solution most relevant for learning lie along the dominant eigenvectors of $K$, 
which correspond to the dominant spectral basis functions of the reproducing kernel Hilbert space (RKHS) $\Hc$ induced by kernel $k$ with respect to $X$.
\cite{lin2023sampling} shows that SGD enjoys an improved early convergence along these directions, but attains a slow sublinear rate asymptotically.
% On the other hand, PCG converges at an asymptotically faster linear rate, but suffers from scalability issues due to a slow $\bigO(n^2)$ per-iteration cost. 
Motivated by these observations, we study the convergence rate of tail-averaged \sap{} along the dominant spectral basis functions, 
showing that it obtains a two-phase convergence rate: an initial condition number-free sublinear rate, 
followed by an asymptotic linear rate where the condition number dependence is mitigated by the blocksize $b$.
The proofs of the results in this section appear in \cref{s:proof-gp}.

Following \cite{lin2023sampling}, we define the spectral basis functions of $\Hc$ as 
\begin{align*}
u^{(i)} = 
\argmax_{u\in\Hc,\|u\|_{\Hc}=1}\Big\{\,\|u(X)\|^2\ \mid \
\langle u,u^{(j)}\rangle_{\Hc} = 0 \ \text{ for }\ j<i\,\Big\}.
\end{align*}
The spectral basis functions $u^{(1)},u^{(2)},...$ are defined in such a way that each of them takes maximal values at the observation points while being orthogonal to the previous ones. 
Hence, to accurately estimate the posterior mean in the data-dense regions, it is most important to minimize the error along the dominant basis functions.

We focus on posterior mean estimation, which corresponds to solving the linear system $K_\lambda W = y$ using tail-averaged \sap{}.
% Here, the update rule is
% \begin{align}
% \iterw{t + 1}{} &= \iterw{t}{} - \Ib^T (\Kbb + \lambda I)^{-1} (\Kbn \iterw{t}{} + \lambda \Ib \iterw{t}{} - y_\B) .\label{eqn:rcd_update}
% %   \text{or equivalently}  \ K_\lambda^{1/2}(& w_{t+1} - \wstar) = (I - \projS) K_\lambda^{1/2}(\iterw{t}{} - \wstar), 
% \end{align}
% where $\projS = K_\lambda^{1/2} S^T (S K_\lambda S^T)^{\dagger} S K_\lambda^{1/2}$. 
Recall that, given a sequence of weight vectors $\iterw{1}{},...,\iterw{t}{}$ produced by \sapfull{} (which are vectors, not matrices, as $\nrhs=1$), tail averaging obtains a weight estimate by averaging the second half of this sequence. 
We use these weights to define an estimate of the GP posterior mean $m_n$, which is given by $\hat m_t = k(\cdot , X) \bar W_t$.
% \begin{align*}
%     \hat m_t = k(\cdot , X) \bar W_t, \quad \text{where}\quad \bar 
%     W_t = \frac{2}{t} \sum_{j=t/2}^t \iterw{j}{}.
% \end{align*}

\cref{theorem:main} is our main theoretical result. 
It provides a convergence guarantee for the \sap{} posterior mean estimate $\hat m_t$ along the \topl{} spectral basis functions, 
where $\ell$ is a parameter that can be chosen freely. 
\cref{theorem:main} shows that while the estimate converges to the true posterior mean over the entire RKHS $\Hc$, its initial convergence rate along the dominant basis functions is even faster.
Our theoretical results rely on the \textit{smoothed condition number}, which we define below.

\begin{definition}[Smoothed condition number]
\label{def:smoothed_condition_number}
Let $A \in \R^{n \times n}$ be a positive-semidefinite matrix with eigenvalues $\lambda_1 \geq \ldots \geq \lambda_n$.
We define the smoothed condition number $\phi(b, p)$ of $A$ as
\[
    \phi(b, p) \coloneqq \frac{1}{b} \sum_{i > b} \frac{\lambda_i}{\lambda_p}.
\]
\end{definition}

The smoothed condition number controls both the SGD-style sublinear convergence rate and the PCG-style linear convergence rate in \cref{theorem:main}. Note that, for notational convenience, we use $2b$ to denote the \sap{} blocksize in the below statement.

% \mnote{Depending on space, it might be good to discuss what these top spectral basis functions represent for GP inference.}

% \mnote{Alternate version of the theory, which allows more flexibility in the block size.}

% \pnote{Let's add some intuition for the result based on \cite{lejeune2024asymptotics}}

\begin{theorem}\label{theorem:main}
    Suppose we have a kernel matrix $K \in \R^{n \times n}$, observations $y \in \R^n$ from a GP prior with likelihood variance $\lambda \geq 0$, and let $\phi(\cdot,\cdot)$ denote the smoothed condition number of $K + \lambda I$. 
    Given any $\ell\in[n]$, let $ \rmproj_\ell$ denote orthogonal projection onto the span of $u^{(1)},...,u^{(\ell)}$. 
    For any $1\leq b\leq n/2$, in $\bigOt(nb^2)$ time we can construct a distribution over  row index subsets $\B$ of size $2b$ such that \sap{} (Algorithm~\ref{alg:sap}) initialized at zero after $t$ iterations satisfies
    \begin{align*}
        \E \left\|  \rmproj_\ell(\hat m_t) -  \rmproj_\ell(m_n) \right\|_{\Hc}^2
        &\leq \min \left\{ \frac {8\phi(b,\ell)}t, \left( 1 - \frac{1}{2\phi(b,n)} \right)^{t/2} \right\} \|y\|_{K_\lambda^{-1}}^2.
        % \\
        % \text{where}\quad \phi(b,p) &:= \frac2b\sum_{i>b/2}\frac{\lambda_i+\lambda}{\lambda_p+\lambda}.
    \end{align*}
    % where $\bar\kappa(b,p) = \frac1b\sum_{i>b/2}\frac{\lambda_i+\lambda}{\lambda_p+\lambda}$.
    % where $\phi(b,k) = \bigO(\frac1b\sum_{i>b/2}\frac{\lambda_i+\lambda}{\lambda_k+\lambda}\big)$ and $\kappa(b)= \bigO\big(n/b + \lambda_{b/2}\lambda\big)$.
\end{theorem}

This result provides a composite convergence rate for \sap{}: the first rate, $\frac{8\phi(b,\ell)}t$, is a sublinear rate, similar to SGD \citep{lin2023sampling}, while the second rate, $\left( 1 - \frac1{2\phi(b,n)} \right)^{t/2}$, is a linear convergence rate. 
The linear rate is independent of the subspace norm, and thus, in particular, also applies for $\ell = n$, i.e., when comparing posterior means within the entire space. 
However, for $\ell \ll n$, the sublinear rate wins out in the early phase of the convergence, since $\phi(b,\ell) \ll \phi(b,n)$. 

% The function $\phi(b,p)$ controls both rates, and can be viewed as a smoothed condition number quantity, with the second parameter controlling the denominator $\lambda_p+\lambda$. 
For ill-conditioned kernels with small prior variance $\lambda$, the linear rate $1 - 1/ (2 \phi(b,n))$ is close to 1 (as $b$ increases, the preconditioning effect becomes stronger, which improves the rate). 
On the other hand, the constant in the sublinear rate $\phi(b,\ell)$ is much smaller, since it takes advantage of the \topl{} subspace norm, which makes it effectively independent of the conditioning of $K$. 

If the \sap{} blocksize is proportional to $\ell$ and $K + \lambda I$ exhibits \emph{any} polynomial spectral decay, 
then $\phi(b,\ell) = \bigO(1)$, making the sublinear rate condition number-free. 
Most popular kernels exhibit polynomial spectral decay
\citep{ma2017diving,kanagawa2018gaussian}.
Furthermore, polynomial spectral decay is a standard assumption in the generalization analysis of kernel methods in the fixed and random design settings \citep{caponnetto2007optimal,bach2013sharp,rudi2015less}.  
To the best of our knowledge, our result is the first of its kind, as previous comparable guarantees for stochastic methods \citep[][Proposition 1]{lin2023sampling} still depend on the condition number. 
We note that our condition number-free rate is attained only when seeking moderately accurate estimates of the posterior mean (which is often the case due to the presence of noise in the data). 
When seeking highly accurate estimates, the asymptotic (condition number dependent) linear rate will eventually take over.
We illustrate these claims with the following corollary:
%which gives the time complexity for obtaining $\epsilon$ relative error with respect to the subspace norm $\| \rmproj_\ell(\cdot)\|_{\Hc}$ using \sapfull{}.
\begin{corollary}\label{corollary:main}
    Suppose that the matrix $K$ exhibits polynomial spectral decay, i.e., $\lambda_i(K) = \Theta(i^{-\beta})$ for some $\beta > 1$. 
    Then for any $\ell \in \{1,\ldots,n\}$ and $\epsilon\in(0,1)$, choosing block size $b=4\ell$ we can find $\hat m$ that with probability at least $0.99$ satisfies $\| \rmproj_\ell(\hat m) -  \rmproj_\ell(m_n)\|^2_{\Hc} \leq \epsilon \cdot \| \rmproj_\ell(m_n)\|^2_{\Hc}$ in
    \begin{align*}
    \bigOt \left( (n^2 + n\ell^2) \min \left\{ \frac{1}{\epsilon}, \left( 1 + \frac{\ell(\lambda_\ell(K)+\lambda)}{n(\lambda_n(K)+\lambda)} \right)\log \left( \frac{1}{\epsilon} \right) \right\} \right) \quad \text{time}.
    \end{align*}
\end{corollary}
\begin{remark}\label{remark:main}
    The $\tilde \bigO( (n^2 + n \ell^2)/ \epsilon)$ runtime (phase one, attained for moderate accuracy $\epsilon$) is entirely independent of the condition number of $K$ and the likelihood variance $\lambda$. 
    In \cref{sec:proofs}, we show that when $\lambda$ is sufficiently small (the highly ill-conditioned case), then the phase one complexity of \sapfull{} can be further improved by increasing the block size $b$ to attain $\tilde \bigO(nb^2+(n^2+nb^2)(\ell/b)^{\beta-1}/\epsilon)$ runtime. 
    In particular, when $\ell \ll b \ll \sqrt n$, we obtain meaningful convergence in $o(n^2)$ time, i.e., before reading all of the data. 
    Thus, the \topl{} subspace convergence of \sap{} can actually benefit from ill-conditioned kernels as the leading spectral basis functions become more~dominant.
\end{remark}

Although our theoretical results require tail averaging, we do not believe it is needed for good practical performance of sketch-and-project algorithms.
Indeed, we run sketch-and-project with and without tail averaging in \cref{subsec:tail_avg_appdx}, and we find that (i) tail averaging does not improve subspace convergence by a substantial amount and (ii) tail averaging performs worse than not using tail averaging at larger blocksizes.

\section{\adasap{}: Approximate, Distributed, Accelerated \sap{}}
\label{sec:adasap}
We introduce \adasap{} (\cref{alg:adasap}), a scalable extension of \sap{} for GP inference. 
We first introduce the modifications made to \sap{} for scalability (\cref{subsec:adasap_ingredients}), before presenting \adasap{} in full (\cref{subsec:adasap_alg}).
For a detailed discussion of how \adasap{} relates to prior work, please see \cref{sec:related_work}.

% \unote{This paragraph needs a rewrite. You could say, the theory in section 3 motivates our algorithm ADASAP. The major differences are X, Y, Z in the theoretical version and A, B, C in the practical version. Here is why we needed the weird theoretical tricks (eg, tail averaging) and here's why we don't think it's needed in practice. Here's why we think the tricks we used in practice would still allow provable convergence.}

% \adasap{} uses several techniques not covered by our theory in \cref{theorem:main}, which only applies to tail-averaged \sap{}.
% Nonetheless, based on the analysis developed in this work and prior work, 
% we believe the strong theoretical guarantees of \cref{theorem:main} can be extended to \adasap{}. 
% h. 

\subsection{The key ingredients: approximation, distribution, and acceleration}
\label{subsec:adasap_ingredients}
We begin by discussing the essential elements of the \adasap{} algorithm: (i) approximate subspace preconditioning, (ii) distribution, (iii) acceleration, and how they enhance performance.

% \subsubsection{Approximate subspace preconditioning}
\subsubsection{Approximate subspace preconditioning}
\sap{} enjoys significant improvements over PCG, as it only requires $\bigO(bn + b^3)$ computation per-iteration.
% Unfortunately, when $n$ is very large, \sap{} still faces scalability issues. 
Unfortunately, \sap{} is limited in how large a blocksize $b$ it may use, as factoring $\Kbb + \lambda I$ to perform subspace preconditioning costs $\bigO(b^3)$. 
This is problematic, as \cref{theorem:main} shows the convergence rate of \sap{} improves as $b$ increases.
%This is intuitive, 
% as in the worst-case, we must make at least one pass through $\Klambd$ to solve $\Klambd W =Y$.
% Intuitively, using a larger $b$ corresponds to preconditioning along a larger subspace, yielding a informative search direction and faster convergence. 

To address this challenge, \adasap{} draws inspiration from previous works \citep{erdogdu2015convergence, frangella2024sketchysgd, rathore2025have}, 
which replace exact linear system solves in iterative algorithms with inexact solves based on low-rank approximations.
\adasap{} replaces $\Kbb$ in \sap{} with a rank-$r$ randomized Nystr{\"o}m approximation $\Knys$ \citep{williams2000using, gittens2016revisiting,tropp2017fixed}.
This strategy is natural, as kernels exhibit approximate low-rank structure \citep{bach2013sharp,rudi2015less, belkin2018approximation}.
Indeed, \cite{rathore2025have} develops an approximate SAP solver for kernel ridge regression with one right-hand side, based on approximating $\Kbb$ by $\Knys$, and observes strong empirical performance.

\adasap{} computes $\Knys$ following the  numerically stable procedure from \cite{tropp2017fixed}, which is presented in \cref{sec:add_alg_deets}.
The key benefit is that $(\Knys + \rho I)^{-1}$ can be applied to vectors in $\bigO(br)$ time (\cref{sec:add_alg_deets}), where $\rho > 0$ is a damping parameter to ensure invertibility.
This reduces the cost compared to the $\bigO(b^3)$
exact SAP update and allows \adasap{} to use larger blocksizes: on the taxi dataset ($n = 3.31 \cdot 10^{8}$), \adasap{} uses blocksize $b = 1.65 \cdot 10^{5}$.
% This procedure returns $U \in \R^{n\times r}$ with orthonormal columns, and a diagonal matrix $S \in \R^{r\times r}$ such that $\Knys = USU^{T}$.
 
% As $\Knys$ takes the form of an eigendecomposition, 
% On the other hand, solving a linear system with size $1.65 \cdot 10^5 \times 1.65 \cdot 10^5$ at every iteration would be prohibitively expensive in terms of compute and memory. 

%\mnote{Maybe use different font for dataset names like "taxi".}

% \mnote{It feels like Nystr\"om approximation should be defined here, especially since $\Omega$ is referenced later.}

% \begin{figure}
%     \centering
%     \includegraphics[width=0.5\linewidth]{figs/parallel_scaling/taxi.pdf}
%     \caption{Multi-GPU scaling of \adasap{} on the taxi dataset. 
%     \adasap{} obtains near-linear scaling with the number of GPUs.}
%     \label{fig:parallel_scaling_taxi}
% \end{figure}

% \subsubsection{Distributed matrix-matrix products}
\subsubsection{Distributed matrix-matrix products}
\sap{} with approximate subspace preconditioning allows large blocksize $b$, but two bottlenecks remain: computing $\Kbn W_t$ and constructing the sketch $\Kbb \Omega$. 
The former costs $\bigO(bn)$, while the latter costs $\bigO(b^2r)$.
As matrix-matrix multiplication is embarrassingly parallel, distributed multi-GPU acceleration can address these bottlenecks. 
By partitioning $\Kbn$ across $\nworks$ GPUs, \algcoldist{} significantly reduces the time to compute the product $\Kbn \Omega$.
% \unote{not the cost, only the predicted time. But I'd suggest you omit these two sentences, and just say that the operations are perfectly parallelizable, and in practice we observe excellent parallel scaling of 2.6$\times$ from 3 workers.} from $\bigO(bn\nrhs)$ to $\bigO(b n \nrhs/\nworks)$. 
Similarly, \algrowdist{} for $\Kbb \Omega$ reduces the time to compute the sketch.
On our largest dataset, taxi, using 4 GPUs achieves a $3.4\times$ speedup (\cref{fig:parallel_scaling_taxi}), which corresponds to near-perfect parallelism.

% \sap{} with approximate subspace preconditioning removes the barrier to using a large blocksize $b$ but two major bottlenecks remain: 
% (i) computing the product $\Kbn W_t$ and (ii) constructing $\Knys$.
% The dominant cost per iteration is in computing
% $\Kbn W_t$, which becomes expensive when both 
% $n$ and $b$ are large (e.g., $n > 10^8, b > 10^5$), even when using GPU acceleration. 
% The second computational bottleneck is calculating $\Knys$, as computing the sketch $\Kbb\Omega$ has cost $\bigO(b^2r)$.

% Fortunately, matrix-matrix multiplication can easily be sped-up by using multiple GPU devices. 
% If $\nworks$ GPUs are available, we can partition of $\Kbn$ over the column dimension and distribute the product by sending a partition to each worker.
% This computation is performed in the \algcoldist{} routine (\cref{alg:col_dist_mat_mat})
% Distribution reduces the cost to $\bigO(b n \nrhs/\nworks)$ from $\bigO(bn\nrhs)$.
% \adasap{} uses an analogous routine \algrowdist{} (\cref{alg:row_dist_mat_mat}) for $\Kbb \Omega$, reducing the cost from $\bigO(b^2r)$ to $\bigO(b^2r/\nworks)$. 
% \cref{fig:parallel_scaling_taxi} shows the benefits of using multiple GPUs on our largest example taxi, where distributing across 3 GPUs leads to a $2.6\times$ speedup, 
% which reduces the time to make a single pass from 3.9 days to 1.5 days.

\subsubsection{Nesterov acceleration}
% \adasap{} uses Nesterov acceleration to speed up convergence. 
Prior work has shown  Nesterov acceleration \citep{nesterov1983method} improves the convergence rate of \sap{} and approximate \sap{} \citep{tu2017breaking, gower2018accelerated, derezinski2024fine, rathore2025have}.
Hence we use Nesterov acceleration in \adasap{} to improve convergence.
% In particular, \cite{derezinski2024fine} shows that acceleration reduces the dependence of the convergence rate on the smoothed condition number (\cref{def:smoothed_condition_number}) by a square-root factor.
% This reduction leads to significant gains when the condition number (even after subspace preconditioning) is much larger than 1.

\subsection{\adasap{} algorithm}
\label{subsec:adasap_alg}
\begin{algorithm}
\caption{\adasap{} for $\Klambd W = Y$}
    \label{alg:adasap}
    % \tiny
    \begin{algorithmic}
        \Require distribution of row indices $\D$ over $[n]^{b}$,
        distribution over \nys{} sketching matrices $\Dnys$ over $\R^{b \times r}$,
        number of iterations $T_{\max}$, 
        initialization $\iterw{0}{}\in \R^{n\times \nrhs}$, 
        workers $\{\W_1, \ldots, \W_{\nworks}\}$, 
        % column block oracles $(\Kcol_{\W_{1}},\ldots, \Kcol_{\W_{\nworks}})$, 
        % row block oracles $(\Krow_{\W_{1}}, \ldots,  \Krow_{\W_{\nworks}})$, 
        averaging boolean tail\_average,
        acceleration parameters $\mu$, $\nu$
        \State\State{\textbf{\# Initialize acceleration parameters}}
        \State
        $\beta \gets 1 - \sqrt{\mu / \nu}, \quad \gamma \gets 1/\sqrt{\mu \nu}, \quad  \alpha \gets 1/(1 + \gamma \nu)$
        \State $V_0\gets W_0$, $Z_0 \gets W_0$
        \State
        \For{$t = 0, 1, \ldots, T_{\max} - 1$}
        \State{\textbf{\# Step I: Compute matrix-matrix product $\Kbn \iterw{t}{}$}}  \Comment{Costs $\bigO(bn\nrhs/\nworks)$}
        \State Sample row indices $\B$ of size $b$ from $[n]$ according to $\D$
        \State $\Kbn \iterw{t}{} \gets \algcoldist(\iterw{t}{}, \B, \{\W_1, \ldots, \W_{\nworks}\})$
        \Comment{\cref{alg:col_dist_mat_mat}}
        \State 
        % \State Partition rows of $W$ as $\left\{(W_t)_1, \ldots, (W_t)_{\nworks}\right\}$ and send $(W_t)_i$ to $\W_i$
        % \State Compute in parallel across workers $(\W_1, \ldots, \W_{\nworks})$
        %      \State $(\Kbn \iterw{t}{})_{i} \gets \Kcol_{\W_i}[(\iterw{t}{})_i]$ \Comment{Compute column block product}
        % \State Aggregate $\Kbn \iterw{t}{} \gets (\Kbn \iterw{t}{})_{1} + \ldots +(\Kbn \iterw{t}{})_{\nworks}$
        \State{\textbf{\# Step II: Compute \nys{} sketch $\Kbb \Omega$}} \Comment{Costs $\bigO(rb^2/\nworks$)}
        \State Sample $\Omega \in \R^{b \times n}$ from $\Dnys$
        % \State Compute in parallel across workers $(\W_1, \ldots, \W_{\nworks})$
        % \State
        %     $(\Kbb \Omega)_i \gets \Krow_{\W_{i}}[\Omega]$ \Comment{Compute row block product}
        % \State Aggregate $\Kbb \Omega \gets \left[(\Kbb \Omega)_1^T~ \ldots ~ (\Kbb \Omega)_\nworks^T \right]^T$
        \State $\Kbb \Omega \gets \algrowdist(\Omega, \B, \{\W_1, \ldots, \W_{\nworks}\})$
        \Comment{\cref{alg:row_dist_mat_mat}}
        \State 
        
        \State \textbf{\# Step III: Compute Nystr{\"o}m approximation and get stepsize} \Comment{Costs $\bigO(b^2 + br^2)$}
        \State $U, S \gets \algnys(\Kbb\Omega, \Omega, r)$ \Comment{\cref{alg:nystrom}}
        \State $P_\B \gets USU^{T} + (S_{rr}+\lambda) I$ \Comment{Get preconditioner. Never explicitly formed!}
        \State $\eta_\B \gets \algstpsz(P_\B, \Kbb + \lambda I)$ \Comment{\cref{alg:get_stpsz}}
        \State
        
        \State \textbf{\# Step IV: Compute updates using acceleration}
        \Comment{Costs $\bigO(b r \nrhs + n \nrhs)$}
        \State $\iterd{t}{} \gets \Ib^T P_{\B}^{-1} (\Kbn \iterw{t}{} + \lambda \Ib W_t - Y_\B)$ % \Comment{Compute search direction: costs $\bigO(br\nrhs)$}
        \State $\iterw{t + 1}{}, \iterv{t + 1}{}, \iterz{t + 1}{} \gets \algnestacc(\iterw{t}{}, \iterv{t}{}, \iterz{t}{}, \iterd{t}{}, \eta_\B, \beta, \gamma, \alpha)$
        \Comment{\cref{alg:nest_acc}}
        \State
        % \State $\iterw{t+1}{} \gets \iterz{t}{} -\eta_B \iterd{t}{}$ \Comment{Update parameters: costs $\bigO(b\nrhs)$}
        % \State $\iterv{t+1}{} \gets \beta \iterv{t}{} + (1-\beta)\iterz{t}{}-\gamma \eta_B \iterd{t}{}$
        % \Comment{Update velocity: costs $\bigO(n\nrhs)$}
        % \State $\iterz{t+1}{} \gets \alpha \iterv{t}{} + (1 - \alpha)\iterw{t+1}{}$ \Comment{Update Nesterov iterate: costs $\bigO(n\nrhs)$}
        \If{tail\_average}
        \State $\bar W_{t + 1} \gets 2 / (t+1) \sum_{j = (t+1)/2}^{t} W_j$
        \EndIf
        \EndFor
        \State \Return $\bar W_T$ if tail\_average else $\iterw{T}{}$ \Comment{Approximate solution to $\Klambd W = Y$}
    \end{algorithmic}
\end{algorithm}
The pseudocode for \adasap{} is presented in \cref{alg:adasap}.
\adasap{} uses tail averaging, approximate subspace preconditioning, distributed matrix-matrix products, and Nesterov acceleration.
Assuming perfect parallelism, 
\adasap{} has per iteration runtime of $\bigO(n\nrhs b/\nworks + b^2r/\nworks)$, a significant improvment over the $\bigO(n\nrhs b + b^3)$ iteration time of \sap{}.
The per-iteration runtime of \adasap{} is comparable to distributed SDD---in other words, \adasap{} effectively preconditions the problem, reducing the total iterations required, without significantly slowing down each iteration.
% Thus, \adasap{} can better tackle the inherent ill-conditioning of kernel linear systems through preconditioning and acceleration without sacrificing speed. 
Moreover, unlike SDD, 
\adasap{} automatically sets the stepsize at each iteration, 
removing the need for expensive tuning.
By default, we set the acceleration parameters $\mu$ and $\nu$ to $\lambda$ and $n / b$, respectively (as done in \cite{rathore2025have}), and find they
yield excellent performance in \cref{sec:experiments}.

\subsubsection{Theory vs. Practice}
% Relative to \sap{}, \adasap{} uses approximate subspace preconditioning, acceleration, and distribution.
% Approximate subspace preconditioning and distribution allow for larger blocksizes and leads to faster runtimes.
% \adasap{} uses acceleration, as it is known to improve convergence on ill-conditioned problems.
Our theory in \cref{sec:sap} does not cover the approximate preconditioning, acceleration, or uniform sampling that is used in \adasap{}.
% Moreover, \adasap{} uses uniform sampling, while our theoretical results require sampling from a DPP.
Despite the theory-practice gap between \sap{} (\cref{alg:sap}) and \adasap{} (\cref{alg:adasap}), we believe the theory developed for \sap{} can be extended to \adasap{}.
This belief is rooted in \cite{derezinski2024solving, derezinski2024fine, rathore2025have}, which show that \sap{} methods still converge when using approximate preconditioning, acceleration, and uniform sampling. 
We leave this extension as a direction for future research.

For simplicity, our experiments run \adasap{} without tail averaging.
Tail averaging is needed to establish \cref{theorem:main}, but we do not expect it to yield significant practical improvements (\cref{subsec:tail_avg_appdx}).
This is in line with the SGD literature, where averaging leads to better theoretical convergence rates, but the last iterate delivers similar performance in practice \citep{shamir2013stochastic, johnson2013accelerating}.

\section{Experiments}
\label{sec:experiments}
We present three sets of experiments showing \adasap{} outperforms state-of-the-art methods for GP inference:
(i) GP inference on large benchmark datasets (\cref{subsec:gp_inf}), 
(ii) GP inference on huge-scale transportation data analysis with $n > 3 \cdot 10^8$ samples (\cref{subsec:gp_inf_taxi}), and (iii) the Bayesian optimization task from \cite{lin2023sampling} (\cref{subsec:bayes_opt}).
We evaluate \adasap{} against the following competitors:
\begin{itemize}
    \item \adasapi{}: A variant of \adasap{} where the subspace preconditioner $P_\B$ is set to the identity matrix.
    \adasapi{} is the same  as accelerated randomized block coordinate descent.
    \item \sdd: The coordinate descent method introduced by \cite{lin2024stochastic}.
    We tune \sdd{} using three different stepsizes, and denote these variants by \sdd-1, \sdd-10, and \sdd-100.
    \item \pcg: A combination of block CG \citep{oleary1980block} with \nys{} preconditioning \citep{frangella2023randomized}.
\end{itemize}
% We find that \adasap{} consistently outperforms \sdd{} and \pcg{} in all three sets of experiments.

Our experiments are run in single precision on 48 GB NVIDIA RTX A6000 GPUs using Python 3.10, PyTorch 2.6.0 \citep{paszke2019pytorch}, and CUDA 12.5.
We use 2, 3, and 1 GPU(s) per experiment in \cref{subsec:gp_inf,subsec:gp_inf_taxi,subsec:bayes_opt}, respectively.
Code for reproducing our experiments is available at \href{\codeurlpublic}{\codeurlpublicdisplaytext}.

Additional details are in \cref{sec:experiments_appdx}.

\subsection{GP inference on large-scale datasets}
\label{subsec:gp_inf}
We benchmark on six large-scale regression datasets from the UCI repository, OpenML, and sGDML \citep{chimela2017machine}.
The results are reported in \cref{tab:gp_inf}.
\adasap{} achieves the lowest RMSE and mean negative log-likelihood (NLL) on each dataset; the NLL is computed using 64 posterior samples (via pathwise conditioning).
\sdd-10 is competitive with \adasap{} on yolanda and acsincome, but performs much worse on the other datasets.
\sdd{} is also sensitive to the stepsize: \sdd-1 attains a larger RMSE and NLL than \adasap{} on all datasets, while \sdd-100 diverges on all datasets.
\pcg{} obtains the same RMSE and NLL as \adasap{} on yolanda and song, but its performance degrades for larger datasets.
Additionally, \cref{fig:all_main_time} shows \adasap{} achieves the lowest RMSE and NLL throughout the optimization process, demonstrating its efficiency with respect to wall-clock time.

\begin{table}
    \centering
    \small
    \caption{RMSE and mean negative log-likelihood (NLL) obtained by \adasap{} and competitors on the test set. 
    The results are averaged over five 90\%-10\% train-test splits of each dataset.
    We \textbf{bold} a result if it gets to within $0.01$ of the best found RMSE or mean NLL.
    }
    \label{tab:gp_inf}
    \begin{tabular}{llcccccc}
        \toprule
        & Dataset & yolanda & song & benzene & malonaldehyde & acsincome & houseelec \\
        & $n$ & $3.60 \cdot 10^{5}$ & $4.64 \cdot 10^{5}$ & $5.65 \cdot 10^{5}$ & $8.94 \cdot 10^{5}$ & $1.50 \cdot 10^{6}$ & $1.84 \cdot 10^{6}$ \\
        & $d$ & $100$ & $90$ & $66$ & $36$ & $9$ & $9$ \\
        & $k$ & RBF & \mtrn-3/2 & \mtrn-5/2 & \mtrn-5/2 & RBF & \mtrn-3/2 \\
        \midrule
        \multirow{6}{*}{\rotatebox[origin=c]{90}{Test RMSE}} & \adasap{} & \textbf{0.795} & \textbf{0.752} & \textbf{0.012} & \textbf{0.015} & \textbf{0.789} & \textbf{0.027} \\
         & \adasapi{} & 0.808 & 0.782 & 0.168 & 0.231 & \textbf{0.795} & 0.066 \\
         & \sdd-1 & 0.833 & 0.808 & 0.265 & 0.270 & 0.801 & 0.268 \\
         & \sdd-10 & \textbf{0.801} & 0.767 & 0.112 & Diverged & \textbf{0.792} & 0.119 \\
         & \sdd-100 & Diverged & Diverged & Diverged & Diverged & Diverged & Diverged \\
         & \pcg & \textbf{0.795} & \textbf{0.752} & 0.141 & 0.273 & 0.875 & 1.278 \\
        \midrule
        \multirow{6}{*}{\rotatebox[origin=c]{90}{Test Mean NLL}} & \adasap{} & \textbf{1.179} & \textbf{1.121} & \textbf{-2.673} & \textbf{-2.259} & \textbf{1.229} & \textbf{-2.346} \\
         & \adasapi{} & 1.196 & 1.170 & -0.217 & 0.466 & \textbf{1.235} & -2.185 \\
         & \sdd-1 & 1.225 & 1.203 & 0.531 & 0.903 & 1.242 & -0.281 \\
         & \sdd-10 & \textbf{1.187} & 1.149 & -0.762 & Diverged & \textbf{1.232} & -1.804 \\
         & \sdd-100 & Diverged & Diverged & Diverged & Diverged & Diverged & Diverged \\
         & \pcg & \textbf{1.179} & \textbf{1.121} & -0.124 & 0.925 & 1.316 & 2.674 \\
        \bottomrule
    \end{tabular}
\end{table}

\begin{figure}
    \centering
    \includegraphics[width=\ifarxivpreprint1\else0.9\fi\linewidth]{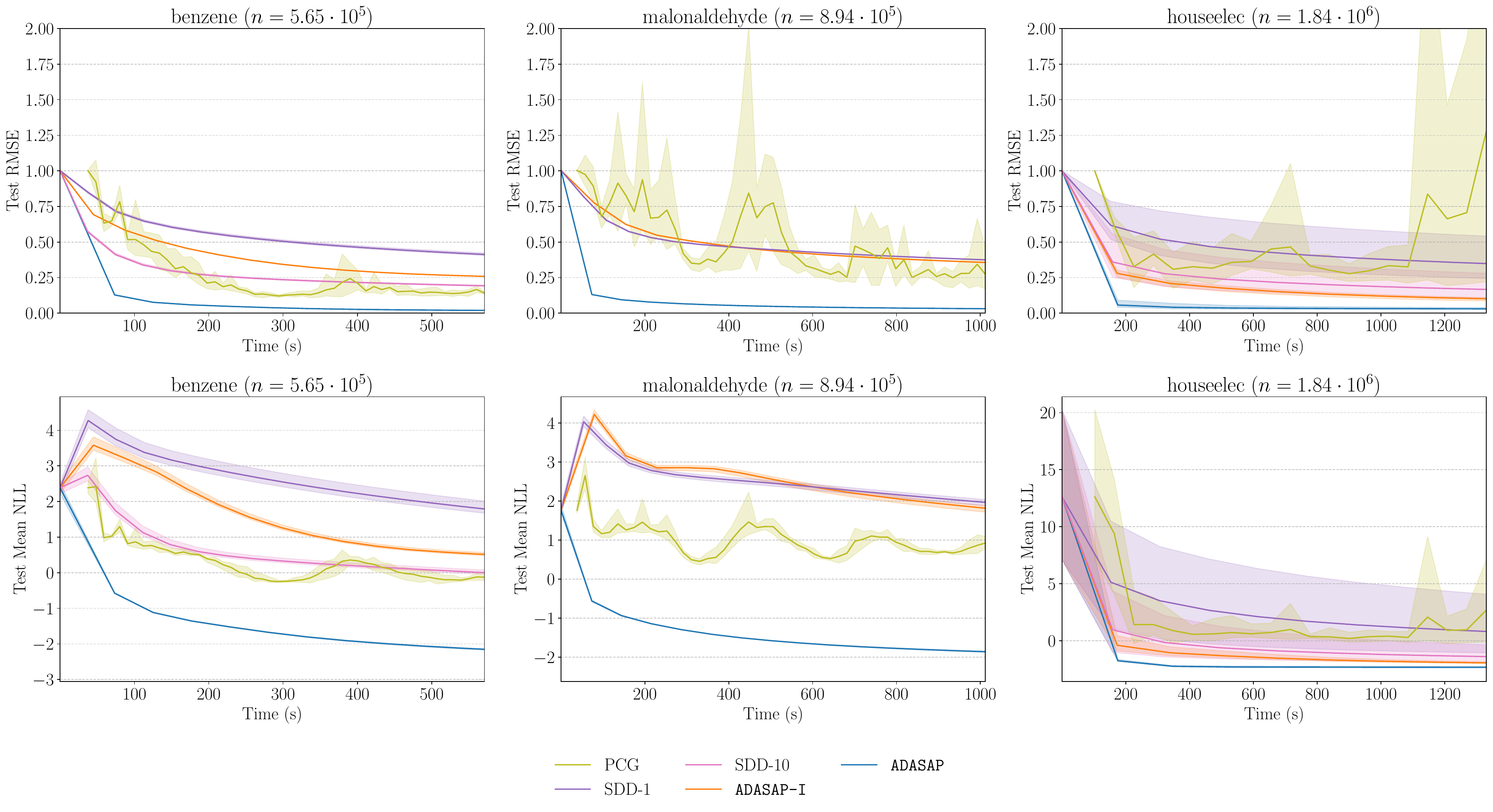}
    \caption{Performance of \adasap{} and competitors on RMSE and mean NLL, as a function of time, for benzene, malonaldehyde, and houseelec.
    The solid curve indicates mean performance over random splits of the data;
    the shaded regions indicate the range between the worst and best performance over random splits of the data.
    \adasap{} outperforms the competition.
    }
    \label{fig:all_main_time}
\end{figure}

\subsection{Showcase: Transportation data analysis with $> 3 \cdot 10^8$ samples}
\label{subsec:gp_inf_taxi}
To demonstrate the power of \adasap{} on huge-scale problems, we perform GP inference on a subset of the taxi dataset (\href{https://github.com/toddwschneider/nyc-taxi-data}{https://github.com/toddwschneider/nyc-taxi-data}) with $n = 3.31 \cdot 10^8$ samples and dimension $d = 9$: the task is to predict taxi ride durations in New York City.
To the best of our knowledge, this is the first time that full GP inference has been scaled to a dataset of this size.
Due to memory constraints, we are unable to compute 64 posterior samples as done in \cref{subsec:gp_inf} for computing NLL, so we only report RMSE.

The results are shown in \cref{fig:taxi_time}.
\adasap{} obtains the lowest RMSE out of all the methods.
Once again, \sdd{} is sensitive to the stepsize: \sdd-1 obtains an RMSE of 0.60, as opposed to \adasap{}, which obtains an RMSE of 0.50, while \sdd-100 diverges.
\sdd-10 obtains an RMSE of 0.52, which is similar to that of \adasap{}.
However, \adasap{} reaches the RMSE attained by \sdd-10 in 45\% less time than \sdd-10, which translates to a difference of 14 hours of runtime!
Furthermore, \pcg{} runs out of memory, as the memory required for storing the sketch in single precision is $3.31 \cdot 10^8 \cdot 100 \cdot 4 \approx 130$ GB, which exceeds the 48 GB of memory in the A6000 GPUs used in our experiments.

\begin{figure}
    \centering
    \includegraphics[width=\ifarxivpreprint0.5\else0.4\fi\linewidth]{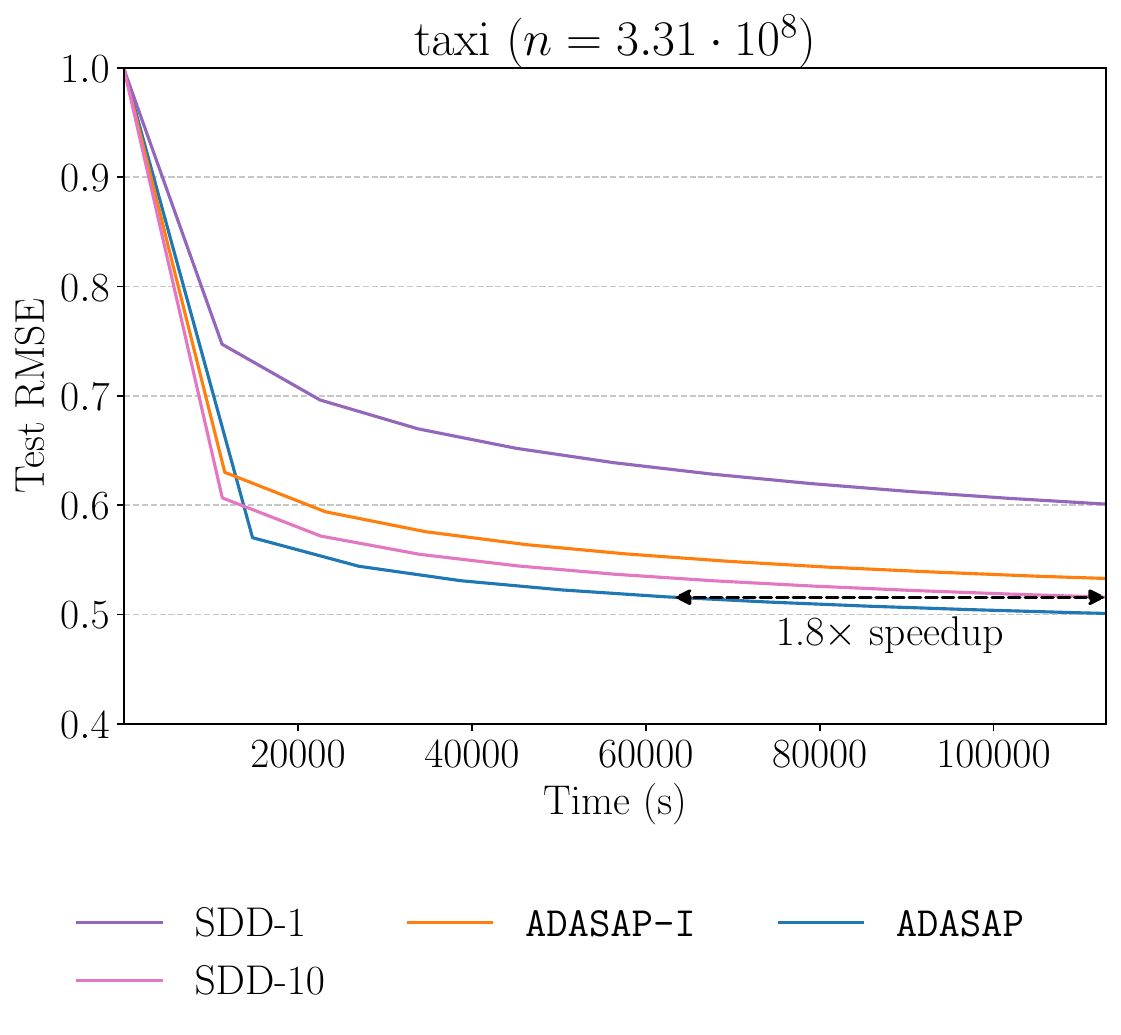}
    \caption{Comparison between \adasap{} and competitors on  transportation data analysis.
    \adasap{} attains the lowest RMSE and it obtains a $1.8\times$ speed up over the second-best method, \sdd-10.
    \sdd-100 diverges and \pcg{} runs out of memory, so they do not appear in the figure.}
    \label{fig:taxi_time}
\end{figure}

\subsection{Large-scale Bayesian optimization}
\label{subsec:bayes_opt}
We run \adasap{} and competitors on a variant of the synthetic large-scale Bayesian optimization tasks from \cite{lin2023sampling}.
These Bayesian optimization tasks consist of finding the maximum of black box functions $f: [0, 1]^8 \rightarrow \R$ sampled from a \mtrn{}-3/2 Gaussian process.
We use two different lengthscales (2.0 and 3.0) and 5 random functions per lengthscale.
To avoid misspecification, we set the kernel of each model to match that of the black box function. 
% Additional setup details are provided in \cref{sec:experiments_appdx}.

The results are shown in \cref{tab:bayes_opt}.
\adasap{} makes the biggest improvement over the random search baseline.
As we have seen in the other experiments, \sdd{} is sensitive to the stepsize: \sdd-10 and \sdd-100 make no progress.
\pcg{} also makes less progress than \adasap{} across both lengthscales.

\begin{table}[h!]
    \centering
    \small
    \caption{Percentage improvement over random search for Bayesian optimization, averaged over five seeds. 
    \sdd-10 and \sdd-100 provide no improvement over random search because they are unstable.}
    \label{tab:bayes_opt}
    \begin{tabular}{llcc}
        \toprule
        & & Lengthscale = 2.0 &  Lengthscale = 3.0 \\
        \midrule
        \multirow{6}{*}{\rotatebox[origin=c]{90}{Improv. (\%)}} & \adasap{} & \textbf{10.42} & \textbf{13.86} \\
        & \adasapi{} & 7.04 & 11.27 \\
        & \sdd-1 & 6.50 & 11.17 \\
        & \sdd-10 & 0.00 & 0.00 \\
        & \sdd-100 & 0.00 & 0.00 \\
        & \pcg{} & 0.13 & 5.54 \\
        \bottomrule
    \end{tabular}
\end{table}
\section{Conclusion}
We introduce \adasap{}, an approximate, distributed, accelerated sketch-and-project method for GP inference.
We demonstrate that \adasap{} outperforms state-of-the-art GP inference methods like \pcg{} and \sdd{} on large-scale benchmark datasets, a huge-scale dataset with $> 3 \cdot 10^8$ samples, and large-scale Bayesian optimization.
Moreover, we show that \sap{}-style methods are theoretically principled for GP inference---we prove that \sap{} is the first efficient, condition number-free algorithm for estimating the posterior mean along the top spectral basis functions.
Future work should extend the theoretical results for \sap{} to \adasap{} and investigate \adasap{} in lower precision (e.g., float16).
\begin{ack}
    We would like to thank Jihao Andreas Lin for helpful discussions regarding this work.
    PR, ZF, SF, and MU gratefully acknowledge support from the National Science Foundation (NSF) Award IIS-2233762, the Office of Naval Research (ONR) Awards N000142212825, N000142412306, and N000142312203, the Alfred P. Sloan Foundation, and from IBM Research as a founding member of Stanford Institute for Human-centered Artificial Intelligence (HAI). 
    MD and SG gratefully acknowledge support from NSF Award CCF-2338655.
\end{ack}

% \section*{References}
\bibliographystyle{plainnat}
\bibliography{references}

%%%%%%%%%%%%%%%%%%%%%%%%%%%%%%%%%%%%%%%%%%%%%%%%%%%%%%%%%%%%

\appendix
\section{Related Work}
\label{sec:related_work}
We review the literature on Gaussian process (GP) inference and sketch-and-project, which are the two key areas that serve as foundations for \adasap{}.

\subsection{GP inference}
Naive GP inference based on solving the linear systems in \eqref{eq:post_eqns} exactly and generating posterior samples via \eqref{eq:standard_posterior_sample} has cubic cost in number of training and test points, making it prohibitively expensive for datasets whose size exceeds $n \sim 10^{4}$.
Given its prominence in scientific computing and machine learning, much work has been to done to scale GP inference to the big-data setting.
Prior work in this area can be (roughly) divided into four approaches: (i) \pcg{}-based inference, (ii) SGD-based inference, (iii) low-dimensional inference, and (iv) variational inference.

\subsubsection{\pcg-based inference} 
A natural approach to scaling GP inference is to replace the exact linear system solves by a direct method with approximate linear system solves by an iterative method.
\cite{gardner2018gpytorch} showed that by utilizing batched \pcg{} and GPU acceleration, GP inference could be scaled to datasets of size $n \sim 10^{5}$.
\cite{wang2019exact} further showed that by distributing the kernel matrix across multiple GPUs, that inference could be applied to datasets of size $n \sim 10^{6}$.
Unfortunately, these approaches still have shortcomings. 
While \cite{gardner2018gpytorch,wang2019exact} show that exact GPs can be scaled to datasets with $n \sim 10^{6}$ by leveraging \pcg{} and GPUs, they do not address the cubic complexity of generating posterior samples in \eqref{eq:standard_posterior_sample}.
Moreover, when $n \sim 10^6$, \pcg{} can become prohibitively expensive from both a memory and computational standpoint, as most preconditioners require $\bigO(nr)$ storage \citep{cutajar2016preconditioning, avron2017faster, diaz2023robust} and each \pcg{} iteration costs $\bigO(n^2)$.

\subsubsection{SGD-based inference}
To address the limitations of \pcg-based inference, \cite{lin2023sampling} proposed to replace \pcg{} with SGD.
This reduces iteration cost from $\bigO(n^2)$ to $\bigO(bn)$, where $b$ is the gradient batchsize.
In addition, they address the challenge of the cubic cost of generating posterior samples by adopting the pathwise conditioning strategy \citep{wilson2020efficiently, wilson2021pathwise}.
\cite{lin2023sampling} show SGD outperforms \pcg{} and Stochastic Variational Gaussian Processes (SVGP) \citep{hensman2013gaussian} on various posterior sampling tasks. 
\cite{lin2024stochastic} further improved upon SGD by introducing Stochastic Dual Descent (SDD).
SDD is a coordinate descent algorithm that leverages the dual formulation of the KRR problem to reduce the dependence on the conditioning of $K$ from $\kappa(K(K+\lambda I))$ to $\kappa(K+\lambda I)$.
In addition, SDD incorporates momentum and geometric averaging to further enhance performance.
Compared to SGD, \pcg{}, and SVGP, \cite{lin2024stochastic} shows that SDD yields the best empirical performance.

While \cite{lin2023sampling, lin2024stochastic} address the scaling limitations of \pcg{} and the issue of efficiently sampling from the posterior, there is still room for improvement.
% the approaches they take are not without their own weaknesses.
In terms of theory, \cite{lin2023sampling} only establishes  that SGD is guaranteed to converge at a reasonable rate along the top eigenvectors of $K$.
Convergence along the small eigenvectors is extremely slow, as the rate in \cite{lin2023sampling} along the $\ell^{th}$ eigenvector depends upon $1/\sqrt{\lambda_i}$.
For SDD, the convergence rate is unknown, as \cite{lin2024stochastic} provides no convergence analysis.
Aside from unsatisfactory convergence analyses, it is known from optimization theory, that in the worst case, the convergence rate of SGD and SDD is controlled by the condition number \citep{nemirovskij1983problem}.
As ill-conditioning is a defining property of kernel matrices, this means SDD and SGD will converge slowly.
Another price incurred by SDD and SGD for their fast performance is the presence of additional hyperparameters.
In particular, the stepsize for these methods must be set appropriately to obtain satisfactory convergence.
Tuning the stepsize, while possible, can become expensive for large-scale inference on datasets like taxi.
% The need to find the right stepsize necessitates the use of hyperparameter tuning which can be expensive for large-scale inference on datasets like taxi.
This is in contrast to \pcg{}, which is stepsize-free.

\adasap{} enjoys the scalability benefits of SDD and SGD, but without their limitations.
It only operates on a batch of data at each iteration, so it avoids the $\bigO(n^2)$ iteration cost of \pcg{}.
By incorporating subspace preconditioning, \adasap{} addresses the ill-conditioning of the kernel matrix to achieve fast convergence.
\adasap{} also comes with reliable default hyperparameters, which obviates the need for costly hyperparameter tuning.  

\subsubsection{Low-dimensional inference} 
When the dimension $d$ of the data matrix $X$ is low ($d \leq 3$), 
GP inference can be made far more efficient by exploiting the structure of the kernel matrix and the low-dimensionality of the data.
Common approaches in this vein include rank-structured matrices \citep{ambikasaran2015fast, minden2017fast}, sparse Cholesky factorization \citep{chen2025sparse}, and kernel interpolation \citep{wilson2015kernel, greengard2025ffgp}.
In most cases these methods enable $\bigO(n)$ training time and inference, a massive improvement over standard GP inference methods. 
This linear time complexity has allowed GP training to scale to datasets with $n$ as large as $10^9$ \citep{greengard2025ffgp}.
Unfortunately, the complexity of these algorithms grows exponentially with the dimension, making them impractical when $d > 3$.

\subsubsection{Variational inference}
To address the challenge of large-scale GP inference, many prior works have focused on developing scalable approximate GP methods.
The most popular methods are based on variational inference \citep{titsias2009variational, hensman2013gaussian, hensman2015scalable, hensman2018variational}.
The most well-known approach is SVGP \citep{hensman2013gaussian}, which selects a set of $m$ inducing points by maximizing an evidence lower bound (ELBO).
The use of inducing points leads to only need to work with an $n\times m$ kernel matrix instead of the full $n\times n$ kernel matrix. 
This leads reduces the cost of training to $\bigO(nm^2+m^3)$,  a significant improvement over exact inference methods when $m$ can be taken to be much smaller than $n$.
% \pnote{Something seems off with the previous sentence}
While SVGP and related algorithms significantly reduce the cost of GPs, they generally exhibit worse performance on inference than methods which make use of the full data \citep{wang2019exact,lin2023sampling,lin2024stochastic}.
Thus, it is advisable to use exact inference whenever possible.

\subsection{Sketch-and-project}
\adasap{} builds off the sketch-and-project framework for solving consistent linear systems. 
Sketch-and-project was first formulated in \cite{gower2015randomized}, who showed that randomized Kaczmarz \citep{strohmer2009randomized}, randomized coordinate descent \citep{leventhal2010randomized}, and randomized Newton \citep{qu2016sdna}, are all special cases of sketch-and-project.
They also established a linear convergence rate for sketch-and-project, which is controlled by the smallest eigenvalue of an expected projection matrix.
However, \cite{gower2015randomized} were unable to provide a fine-grained analysis of sketch-and-project in terms of the spectral properties of the linear system, except for a few special cases.

% Despite the lack of a sharp performance analysis, \cite{gower2015randomized} inspired a flurry of research on sketch-and-project. 
Follow-up work has combined sketch-and-project with techniques such as Nesterov acceleration \citep{tu2017breaking} and extended sketch-and-project to quasi-Newton and Newton type methods \cite{gower2018accelerated, gower2019rsn, hanzely2020stochastic}.
Despite these extensions, a sharp convergence analysis of sketch-and-project remained elusive, even in the original setting of consistent linear systems. 

% In the last several years, a sharp characterizations of sketch-and-project for linear systems has finally arisen.
Recent work has finally given a sharp analysis of sketch-and-project.
Using powerful tools from random matrix theory, \cite{derezinski2024sharp} provided the first sharp analysis of sketch-and-project when the sketching matrix is sub-Gaussian or sparse.
In particular, they were the first to prove that if the matrix exhibits a favorable spectral decay profile, then sketch-and-project exhibits condition number-free convergence. 
\cite{derezinski2024solving} improve upon these results, showing that by employing a pre-processing step, one can avoid using expensive sketching matrices at each step, and simply sample rows uniformly.
They accomplish this by leveraging tools from determinantal point process theory \citep{mutny2020convergence,rodomanov2020randomized, anari2024optimal}.
\cite{derezinski2024fine} further refines these results by incorporating Nesterov acceleration and demonstrating an improved computational complexity relative to a fine-grained analysis of \pcg---that is, an analysis of \pcg{} that goes beyond the worst-case  $\bigOt(\sqrt{\kappa})$ iteration complexity \citep{trefethen1997numerical}. 

In addition to improvements in the analysis of exact sketch-and-project for linear systems, 
significant strides have been made when approximate subspace preconditioning is employed. 
\cite{derezinski2024solving, derezinski2024fine} both use \pcg{} to approximately solve the linear system defined by the subspace Hessian in each iteration of \sap{}, while \cite{derezinski2025randomized} augment that system with Tikhonov regularization to further improve its conditioning.
% Provided the system is solved to appropriate accuracy at each iteration, both works find inexact subspace preconditioning enjoys the same convergence rate as the exact case.
\cite{rathore2025have} consider an approximate sketch-and-project method for solving kernel ridge regression with one right-hand side.
They replace $\Kbb + \lambda I$ in \sap{} with $\Knys +\rho I$, where $\Knys$ is a randomized Nystr{\"o}m approximation of $\Kbb$.
They establish linear convergence for the method, with a rate that is comparable to exact \sap{} when the kernel matrix exhibits an appropriate rate of spectral decay.

Among the above-mentioned works, \adasap{} is closest to \cite{rathore2025have}, but with significant algorithmic differences such as tail averaging, handling multiple right-hand sides, and distributing bottleneck operations over multiple devices. 
On the theoretical side, our analysis focuses on the convergence along the top $\ell$-subspace using exact SAP with tail-averaging. In contrast, \cite{rathore2025have} focus on convergence of approximate sketch-and-project (with and without acceleration) \textit{over the entire space}, and as a result, they are only able to show a fast convergence guarantee under certain strong conditions on the spectral decay (effectively limiting the condition number of the problem).
% Our motivation for considering this is prior works such as 
% \cite{dicker2017kernel,lin2020optimal,lin2023sampling}, which show theoretically and empirically that a low-dimensional top eigenspace of $K + \lambda I$ is most relevant for achieving good generalization.     
% In contrast, 
Consequently, the global convergence analysis of \cite{rathore2025have} does not explain the rapid initial progress that SAP-style algorithms make on test error.
We believe that our two-phase analysis consisting of (i) fast condition number-free sublinear convergence, followed by (ii) a slower global convergence rate dependent upon the subspace condition number, better captures this phenomenon. 
In particular, our result is the first step in developing a systematic analysis of the generalization properties of SAP-style algorithms.

To the best of our knowledge, our work is the first to systematically investigate the convergence rate of \sap{} with tail averaging along the top eigenspace.
\cite{epperly2024randomized} combines tail averaging with randomized Kaczmarz for solving inconsistent linear systems, but they do not investigate the convergence rate along the top spectral subspaces.
\cite{steinerberger2021randomized, derezinski2024sharp} look at the convergence rate of the \emph{expected iterates} along a particular eigenvector. 
However, as noted in \cref{subsec:fst_subspace_conv}, this does not lead to a meaningful convergence rate for the iterates produced by the algorithm, and
this also fails to capture the behavior over an entire subspace. 
Thus, our analysis provides the first concrete characterization of the convergence rate of tail-averaged \sap{} along the top eigenvectors.

\section{Proofs of the Main Results}
\label{sec:proofs}
In this section, we give the detailed proofs for the main theoretical results given in Section \ref{subsec:fst_subspace_conv}, namely Theorem \ref{theorem:main} and Corollary \ref{corollary:main}. We start with an informal overview of the analysis in Section \ref{s:proof-overview}, followed by a formal convergence analysis of \sap{} for solving positive-definite linear systems in Section \ref{s:proof-general}, and finally adapting the results to GP posterior mean estimation in Section \ref{s:proof-gp}.

Throughout, we adopt the shorthand $A_\lambda \coloneqq A + \lambda I$, where $A$ is a square matrix, $\lambda \in \R$, and $I$ denotes the identity matrix of the same size as $A$.

\subsection{Overview of the analysis}
\label{s:proof-overview}
Let $w_\star:=K_\lambda^{-1}y\in\R^n$ be the solution of the linear system defined by $K_\lambda$ and $y$, an let us use $w_t\in\R^n$ to denote the iterates produced by \sapfull{} when solving that system.
To convert from the posterior mean error to the Euclidean norm, we first write the spectral basis functions as $u^{(i)}=\frac{1}{\sqrt{\lambda_i}}k(\cdot,X)v_{i}$, where $(\lambda_i,v_i)$ is the $i^{th}$-largest eigenpair of $K$.  Thus, for any $w\in\R^n$ and $m=k(\cdot,X)w\in\Hc$, we have $\| \rmproj_\ell(m-m_n)\|_{\Hc} = \|Q_{\ell}K^{1/2}(w - w_{\star})\|\leq \|Q_{\ell}K_\lambda^{1/2}(w-w_\star)\|$, where $Q_{\ell}=\sum_{i=1}^\ell v_iv_i^T$ is the projection onto the \topl{} subspace of $K$.

Having converted to the Euclidean norm, we express the \sapfull{} update as a recursive formula for its (scaled) residual vector as follows:
\begin{align*}
    \Delta_{t+1} := K_\lambda^{1/2}(& w_{t+1} - w_\star) = (I - \Pi_{\B}) K_\lambda^{1/2}(w_t - w_\star) = (I-\Pi_{\B})\Delta_t, 
\end{align*}
where $\Pi_{\B} = K_\lambda^{1/2} I_{\B}^T (K_{\B,\B}+\lambda I)^{\dagger} I_{\B} K_\lambda^{1/2}$ is a random projection defined by $\B$. To characterize the expected convergence of \sapfull{}, we must therefore control the average-case properties of $\Pi_{\B}$. We achieve this by relying on a row sampling technique known as determinantal point processes (DPPs) \citep{kulesza2012determinantal,derezinski2021determinantal}. A $b$-$\mathrm{DPP}(K_\lambda)$ is defined as a distribution over size $b$ index sets $\B\subseteq[n]$ such that $\Pr[\B]\propto\det(\Kbb + \lambda I)$. DPPs can be sampled from efficiently \citep{calandriello2020sampling,anari2024optimal}: after an initial preprocessing cost of $\tilde O(nb^2)$, we can produce each DPP sample of size $b$ in $\tilde O(b^3)$ time. 

Prior work \citep{derezinski2024solving} has shown that when $\B$ is drawn according to $2b$-$\mathrm{DPP}(K_\lambda)$, then the expectation of $\Pi_{\B}$ has the same eigenbasis as $K$.
%, and thus commutes with both $K_\lambda$ and $Q_{\ell}$. 
Concretely, we can show that $\bar\Pi := \E\,\Pi_{\B} = V\Lambda V^T$ with $\Lambda = \mathrm{diag}(\bar\lambda_1,...,\bar\lambda_n)$, where $\bar\lambda_i\geq \frac1{1+\phi(b,i)}$ and $V$ consists of the eigenvectors of $K$. This leads to a convergence guarantee for the \emph{expected} residual vector:
\begin{align*}
%    \| \rmproj_{\ell}(\E\,\hat m_{t+1}-m_n)\|_{\Hc}
%     \|Q_{\ell}(\E\,\iterw{t+1}{} - \wstar)\|_{K_\lambda}
\|Q_{\ell}\E\,\Delta_{t+1}\|
    &= \|Q_\ell(I - \bar\Pi)Q_\ell \E\,\Delta_t\|
    \leq (1 - \bar\lambda_\ell)\|Q_\ell \E\,\Delta_t\|,
\end{align*}
where we used that $\bar\Pi$, $K_\lambda$, and $Q_\ell$ commute with each other. This would suggest that we should obtain a fast linear rate $(1-\frac1{1+\phi(b,\ell)})^t$ for the \sapfull{} iterates in the \topl{} subspace norm. However, it turns out that the actual residual does not attain the same convergence guarantee as its expectation because, unlike $\bar\Pi$, the random projection $\Pi_{\B}$ does not commute with $K_\lambda$ or $Q_\ell$. Indeed, 
\begin{align*}
\E\|Q_{\ell}\Delta_{t+1}\|^2 = \Delta_t^T\E \left[ (I-\Pi_{\B})Q_\ell(I-\Pi_{\B}) \right]\Delta_t,
\end{align*}
and even though the matrix $(I-\Pi_{\B})Q_\ell(I-\Pi_{\B})$ is low-rank, its expectation may not be, since $\Pi_{\B}$ does not commute with $Q_\ell$. This means that we have no hope of showing a linear convergence rate along the \topl{} subspace. Nevertheless, we are still able to obtain the following bound:
\begin{align*}
    \E \left[ (I-\Pi_{\B})Q_\ell(I-\Pi_{\B}) \right] \preceq Q_\ell(I-\bar\Pi) + (I-Q_\ell)\bar\Pi(I-\bar\Pi).
\end{align*}
This results in a bias-variance decomposition: the first term is the bias, which exhibits fast convergence in \topl{} subspace, while the second term is the variance, which accounts for the noise coming from the orthogonal complement subspace. We then use tail averaging to decay this noise, which allows us to benefit from the fast convergence of expected iterates through a sublinear rate $\frac{\phi(b,\ell)}{t}$.

Finally, to attain the asymptotically faster linear rate for the residual vectors, we observe that every \topl{} subspace norm is upper-bounded by the full norm: $\|Q_\ell\Delta_t\|\leq \|\Delta_t\|$. Thus, we can effectively repeat the above analysis with $\ell=n$ (since $Q_n=I$), in which case the variance term becomes zero, and we recover a linear rate of the form $\left( 1-\frac{1}{1 + \phi(b,n)} \right)^t$.

\subsection{Convergence along \topl{} dimensional subspace for positive-definite matrices}\label{s:proof-general}
In this section, we consider the general problem of solving a linear system $\nA \nw = \ny$ for an $n \times n$ positive-definite matrix $\nA$. 
We provide theoretical results providing the fast convergence guarantees for the iterate sequence generated by SAP along the \topl{} dimensional subspace of $A$, with \cref{theorem:fast_convg_tark_2} being the main result of this section. 
We use \cref{theorem:fast_convg_tark_2} in the next section to derive the subspace convergence result for posterior mean in GP inference, proving \cref{theorem:main}.

% The fast convergence guarantees along the \topl{} dimensional subspace are obtained by leveraging tail averaging to reduce the noise along the $n - \ell$ dimensional residual subspace. 
\paragraph{Preliminaries and notation} We start by introducing the notation and some results from existing literature for SAP and DPPs.
Let $A$ be an $n\times n$ symmetric positive-definite matrix and let $A = V D V^T$ be its eigendecomposition, where $D$ is diagonal with entries $\lambda_1\geq ...\geq\lambda_n$. For other matrices we use $\lambda_i(M)$ to denote the eigenvalues of $M$. For any vector $v\in\R^n$, define $\|v\|_A := \sqrt{v^TAv}$.
Let $V_\ell \in \R^{n \times \ell}$ consist of the \topl{} orthonormal eigenvectors of $A$ and $Q_\ell = V_\ell V_\ell^T$ be the orthogonal projector corresponding to the \topl{} eigenvectors of $A$. 
We rely on subsampling from fixed-size DPPs \citep{kulesza2012determinantal,derezinski2021determinantal} for our theoretical results. 
Here, for convenience, we consider the block size to be $2b$, in contrast to $b$ as considered in \cref{alg:sap}.
For a fixed $b$ and a positive-definite matrix $A$, we identify $2b$-DPP($A$) as a distribution over subsets of $\{1, \ldots , n\}$ of size $2b$, where any subset $\mathcal{B}$ has probability proportional to the corresponding principal submatrix of $A$, i.e., $\det(A_{\mathcal{B},\mathcal{B}})$. 
The subsampling matrix $S \in \R^{2b\times n}$ corresponding to the set $\B$ is defined as a matrix whose rows are the standard basis vectors associated with the indices in $\B$. For notational convenience, we will sometimes say that $S$ is drawn from $2b$-DPP($A$), or $S\sim 2b$-DPP($A$), keeping the index set $\B$ implicit. 

We use the following result from the literature that provides an exact characterization of the eigenvectors of the expected projection matrix that arises in the analysis of sketch-and-project methods, when using subsampling from fixed-size DPPs.

\begin{lemma}[Expected projection under $2b$-DPP($A$), adapted from Lemma 4.1 \citep{derezinski2024solving}]\label{lemma:fix_size_dpp}
    Let $A \in \R^{n\times n}$ be a symmetric positive-definite matrix with eigenvalues $\lambda_1\geq \lambda_2\geq ...\geq\lambda_n$, and let $1 \leq b < n/2$ be fixed. Let $S \sim 2b$-DPP($A$). 
    Then, we have
    \begin{align*}
    \E[A^{1/2} S^T (S A S^T)^\dagger S A^{1/2}] = V D' V^T,
    \end{align*}
    where $D'$ is a diagonal matrix with $j^{th}$ diagonal entry is lower bounded by $\frac{\lambda_j}{\lambda_j + \frac{1}{b} \sum_{i>b}{\lambda_i}}$.
\end{lemma}
As exact sampling from DPPs is often expensive and requires performing operations as costly as performing the eigendecomposition of $A$ \citep{kulesza2012determinantal}, numerous works have looked at approximate sampling from these distributions \citep{calandriello2020sampling,anari2024optimal}. 
Since we exploit the exact characterization of the eigenvectors of the expected projection matrix, our theory requires a very accurate DPP sampling algorithm.
Fortunately, the Markov Chain Monte Carlo tools developed by \cite{anari2020isotropy,anari2024optimal} allow near-exact sampling from fixed-size DPPs (in the sense that the samples are  with high probability indistinguishable from the exact distribution), while offering significant computational gains over exact sampling.

\begin{lemma}[Sampling from $b$-DPP($A$), adapted from \cite{anari2024optimal}]\label{lemma:dpp_sampling}
Given a positive-definite matrix $A$, there exists an algorithm that draws $t \geq 1$ approximate samples from $2b$-DPP($A$) in time $O(n b^2 \log^4 n + t b^3 \log^3 n)$. 
Furthermore, each drawn sample is indistinguishable from an exact sample from $2b$-DPP($A$) with probability at least $1-n^{-O(1)}$.
\end{lemma}

We now provide the main result of this section, incorporating the above two results in the analysis of sketch-and-project (SAP) along the \topl{} subspace of $A$.  Since throughout this section we focus on a linear system with a single right-hand side, we will use $\nw_t$ instead of $W_t$ to denote the \sap{} iterates, as these are always vectors. 
The update rule for SAP is now:
\begin{align}
&\nw_{t + 1} = \nw_t - \nS^T ( \nS \nA \nS^T)^{-1} \nS (\nA \nw_t - y) \nonumber \\
  \text{or equivalently}  \ \nA^{1/2} (& \nw_{t+1} - \nw_{\star}) = (\nI - \Pi) \nA^{1/2}(\nw_t - \nw_{\star}), \label{eqn:rcd_update}
\end{align}
where $\Pi = \nA^{1/2} \nS^T (\nS \nA \nS^T)^{\dagger} \nS \nA^{1/2}$ and $\nw_{\star} = \nA^{-1}y$. Let $\eproj$ denote $\E\,\Pi$. 

In the rest of the section we prove the following result, which is a slight generalization of Theorem~\ref{theorem:main}.
\begin{theorem}[Fast convergence along \topl{} subspace] \label{theorem:fast_convg_tark_2}
    Let $b < n/2$ be fixed and $\nS \in \R^{2b \times n}$ be a random subsampling matrix sampled from $2b$-DPP($\nA$). Furthermore, let $1 \leq \ell <b$ be also fixed and $\lambda_\ell(\eproj) \geq 2\lambda_n(\eproj)$. For any $t > 1$ define $\bar \nw_t = \frac{2}{t} \sum_{i=t/2}^{t-1}{\nw_i}$ where $\nw_i$ are generated using \eqref{eqn:rcd_update}. Then, 
    \begin{align*}
        \E\| \nA^{1/2} (\bar \nw_t - \nw_\star) \|_{Q_\ell}^2 
        & \leq \left( 1-\frac{1}{1+\phi(b,\ell)} \right)^{t/2} \| \nA^{1/2} (\nw_0 - \nw_\star) \|_{Q_\ell}^2 \\
        & \quad + \frac{8}{t} \phi(b,\ell) \left( 1 - \frac{1}{2 \phi(b,n)} \right)^{t/2} \| \nw_0-\nw_\star \|^2_{\nA},
    \end{align*}
    where $\phi(b,p) = \frac{1}{b} \sum_{i>b} \frac{\lambda_i}{\lambda_p}$. 
\end{theorem}
The assumption $\lambda_\ell(\bar\Pi) \geq 2\lambda_n(\bar\Pi)$ is not restrictive buts lets us provide a cleaner analysis by avoiding the edge case of $\lambda_\ell(\bar\Pi) \approx \lambda_n(\bar\Pi)$. In fact, in this corner case, we can simply rely on the existing SAP analysis from previous works and recover our main result, Theorem \ref{theorem:main}. For completeness, we derive Theorem \ref{t:fast_convg_gp} for posterior GP mean inference in Section \ref{s:proof-gp} without the assumption $\lambda_\ell(\bar\Pi) \ge 2\lambda_n(\bar\Pi)$.
The proof of \cref{theorem:fast_convg_tark_2} appears after the proof of \cref{theorem:convg_sub_tark}. We build towards the proof starting with the following result for SAP \citep{derezinski2024solving}: 
\begin{lemma}[Linear convergence with SAP] \label{lemma:cord_desc}
   Let $\nS \in \R^{2b \times n}$ be a random subsampling matrix sampled from $2b$-DPP($\nA$). Then,
    \begin{align*}
        \E\| \nw_{t+1} - \nw_\star \|_{\nA}^2 \leq (1 - \lambda_n(\eproj)) \E \| \nw_t - \nw_\star \|_{\nA}^2.
    \end{align*}
\end{lemma}

However, in the following result, we show that along the \topl{} eigenspace of $\nA$, the expected iterates can converge at a much faster rate than $1 - \lambda_n(\eproj)$.
\begin{lemma}[Convergence of expected iterates along \topl{} subspace] \label{lemma:convg_subsp_1}
    Let $\nS \in \R^{2b\times n}$ be a random subsampling matrix sampled from $2b$-DPP($\nA$). 
    Then
        \begin{align*}
         \|\nA^{1/2}\E_t[\nw_{t+1}-\nw_\star]\|_{Q_\ell}^2 \leq \left(1-\lambda_\ell(\eproj)\right)^2\| \nA^{1/2} (\nw_t - \nw_\star) \|_{Q_\ell}^2,
    \end{align*}
    where $\E_t$ denotes conditional expectation given $\nw_t$.
\end{lemma}

\begin{proof}
    We have
    \begin{align*}
        Q_\ell \nA^{1/2} (\nw_{t+1} - \nw_{\star}) = Q_\ell (\nI - \Pi) \nA^{1/2}(\nw_t - \nw_{\star}).
    \end{align*}
    Taking expectation on both sides and squaring we get,
    \begin{align*}
         \| \nA^{1/2} \E_t[\nw_{t+1}-\nw_{\star}] \|_{Q_\ell}^2 = (\nw_t - \nw_{\star})^T \nA^{1/2} (\nI -\eproj) Q_\ell (\nI - \eproj) \nA^{1/2} (\nw_t - \nw_{\star}).
    \end{align*}
    As $\nA = V \nD \nV^T$ we have $\Pi = V \nD^{1/2} V^T \nS^T (\nS \nV \nD \nV^T \nS^T)^{\dagger} \nS \nV \nD^{1/2} \nV^T$. 
    Due to \cref{lemma:fix_size_dpp} we know that $\eproj$ and $\nA$ share the same eigenvectors, therefore, $\eproj$ and $Q_\ell$ commute. 
    Consequently,
    \begin{align*}
         \|\nA^{1/2} \E_t[ \nw_{t+1} - \nw_{\star}] \|_{Q_\ell}^2 &= (\nw_t - \nw_{\star})^T \nA^{1/2} \nQ_\ell (\nI - \eproj)^2 \nQ_\ell \nA^{1/2} (\nw_t - \nw_{\star}) \\
        & \leq (1- \lambda_\ell (\eproj))^2 \|\nA^{1/2} (\nw_t - \nw_{\star}) \|^2_{Q_\ell}.
    \end{align*}
\end{proof}

We now analyze convergence along the \topl{} dimensional subspace in L2-norm by considering $\E \| \nA^{1/2} (\nw_{t+1} - \nw_{\star}) \|^2_{Q_\ell}$. 
We have,
\begin{align*}
   \E \| \nA^{1/2} (\nw_{t+1} - \nw_{\star}) \|^2_{Q_\ell} = (\nw_t-\nw_{\star})^T \nA^{1/2}  \E \left[ (\nI - \Pi) Q_\ell (\nI - \Pi) \right] \nA^{1/2} (\nw_t - \nw_{\star}).
\end{align*}
In particular, we need to upper bound $\E \left[ (\nI-\Pi) Q_\ell (\nI-\Pi) \right]$. 
We prove the following lemma:
\begin{lemma}[L2-norm error along \topl{} subspace] \label{lemma:convg_l2}
    Let $\nS \in \R^{2b\times n}$ be a random subsampling matrix sampled from $2b$-DPP($\nA$). 
    Then
    \begin{align*}
        \E\| \nA^{1/2} (\nw_{t+1}-\nw_{\star}) \|^2_{Q_\ell} 
        & \leq (1 - \lambda_{\ell}(\eproj)) \E \| \nA^{1/2}(\nw_t -\nw_{\star}) \|^2_{Q_\ell} + \lambda_{\ell + 1}(\eproj(\nI-\eproj)) \E \|\nA^{1/2} (\nw_t - \nw_{\star}) \|^2_{\nI-Q_\ell}.
    \end{align*}
\end{lemma}

% \begin{corollary}\label{corollary:convg_l2}
%     The following is a straightforward consequence;
%     \begin{align*}
%         \E \| \nA^{1/2} (\nw_{t+1} - \nw_{\star}) \|^2_{Q_\ell} 
%         &\leq (1-\lambda_{\ell} (\eproj)) \E\| \nA^{1/2} (\nw_t -\nw_{\star})\|^2_{Q_\ell} + \lambda_{\ell+1}(\eproj(\nI-\eproj)) \E\| \nA^{1/2} (\nw_t - \nw_{\star}) \|^2.
%     \end{align*}
% \end{corollary}

\begin{proof}
First, we rewrite $\E \left[ (\nI-\Pi) Q_\ell (\nI-\Pi) \right]$ as follows:
    \begin{align}
         \E \left[(\nI-\Pi) Q_\ell (\nI-\Pi) \right] &= Q_\ell-\eproj \nQ_\ell - Q_\ell \eproj + \E[\Pi \nQ_\ell \Pi] \nonumber \\
         &=(\nI - \eproj) Q_\ell - Q_\ell \eproj + \E[\Pi \nQ_\ell \Pi] \label{eqn:convg_l2_l2}
    \end{align}
    Crucially, the first two terms terms in \eqref{eqn:convg_l2_l2} live in the top-$\ell$ subspace, but the third term does not. Nevertheless, we are still able to bound it as follows:
    \begin{align*}
        \E[\Pi \nQ_\ell \Pi] = \E[\Pi (\nI - (\nI -Q_\ell)) \Pi]=\E[\Pi] - \E[\Pi (\nI-Q_\ell) \Pi] \preceq \eproj - \eproj (\nI-Q_\ell) \eproj,
    \end{align*}
    where we used Jensen's inequality in the last relation. Substituting in \eqref{eqn:convg_l2_l2} we get,
    \begin{align}
         \E \left[ (\nI-\Pi) Q_\ell (\nI-\Pi) \right] \preceq (\nI-\eproj) Q_\ell + (\nI-Q_\ell) \eproj (\nI-\eproj). \label{eqn:convg_l2}
    \end{align}
    Using \eqref{eqn:convg_l2} we now upper bound $\E \| \nA^{1/2} (\nw_{t+1} - \nw_{\star}) \|^2_{Q_\ell}$ as
    \begin{align*}
        \E \| \nA^{1/2} (\nw_{t+1} - \nw_{\star}) \|^2_{Q_\ell} 
        & \leq (1 - \lambda_\ell(\eproj)) \E \|\nA^{1/2} (\nw_t-\nw_{\star}) \|^2_{Q_\ell} + \lambda_{\ell+1}(\eproj (\nI-\eproj)) \E\| \nA^{1/2} (\nw_t-\nw_{\star}) \|^2_{\nI-Q_\ell},
        % \\
        % & \leq (1 - \lambda_\ell(\eproj)) \E\| \nA^{1/2} (\nw_t-\nw_{\star})\|^2_{Q_\ell} + \lambda_{\ell+1}(\eproj (\nI-\eproj)) \E\|\nA^{1/2} (\nw_t-\nw_{\star})\|^2\\
        % & \leq (1 - \lambda_\ell(\eproj)) \E\| \nA^{1/2} (\nw_t-\nw_{\star}) \|^2_{Q_\ell} \\
        % & \quad + \lambda_{\ell+1}(\eproj (\nI-\eproj)) (1-\lambda_n(\eproj))^t \| \nA^{1/2} (\nw_{0}-\nw_{\star}) \|^2.
    \end{align*}
    which concludes the proof.
\end{proof}
Unrolling this recursive bound, and combining it with the convergence in the full norm, we obtain the following convergence guarantee for \sap{} iterates without tail averaging.
\begin{corollary}\label{corollary:convg_l2_v2}
Let $\alpha = \frac{1 - \lambda_\ell (\eproj)}{1 - \lambda_n(\eproj)}$. Then,
    \begin{align*}
         \E \| \nA^{1/2} (\nw_{t+1} - \nw_{\star}) \|^2_{Q_\ell} & \leq  (1 - \lambda_\ell (\eproj))^{t + 1} \| \nA^{1/2}(\nw_0 - \nw_{\star}) \|^2_{Q_\ell} + \frac{\lambda_{\ell + 1}(\eproj(\nI - \eproj))}{1 - \alpha}(1 -\lambda_n(\eproj))^t \|\nw_0 - \nw_{\star}\|^2_{\nA}.
    \end{align*}
\end{corollary}

\begin{proof}
Define $\gamma = 1 - \lambda_n(\eproj)$.
%Invoking \cref{lemma:cord_desc} we deduce
% We use \cref{lemma:cord_desc} and denote $\gamma =(1-\lambda_n(\eproj))$ and get
Invoking \cref{lemma:convg_l2} we deduce,
    \begin{align*}
        \E \| \nA^{1/2}(\nw_{t+1}-\nw_{\star}) \|^2_{Q_\ell} 
        &\leq (1 - \lambda_{\ell}(\eproj)) \E \| \nA^{1/2}(\nw_t-\nw_{\star})\|^2_{Q_\ell} + \lambda_{\ell+1}(\eproj(I-\eproj))\E\| \nw_t-\nw_{\star} \|^2_{\nA}.        
    \end{align*}
    Now, recursively applying the tower rule and \cref{lemma:cord_desc} yields
    \begin{align*}
    \E \| \nA^{1/2}(\nw_{t+1}-\nw_{\star}) \|^2_{Q_\ell} & \leq (1 - \lambda_{\ell}(\eproj)) \E\| \nA^{1/2} (\nw_t-\nw_{\star})\|^2_{Q_\ell} + \lambda_{\ell+1} (\eproj (\nI - \eproj)) \gamma^t \| \nw_0 -\nw_{\star} \|^2_{\nA}.
    \end{align*}

Unfolding the preceding display yields,
    \begin{align*}
       \E \| \nA^{1/2} (\nw_{t+1}-\nw_{\star}) \|^2_{Q_\ell} 
       & \leq (1 -\lambda_\ell(\eproj) )^{t+1} \| \nA^{1/2} (\nw_0-\nw_{\star}) \|^2_{Q_\ell} \\
       & \quad + \lambda_{\ell+1}(\eproj (\nI - \eproj)) \left( \sum_{i=0}^{t}{\gamma^{i} (1 - \lambda_\ell(\eproj))^{t-i}} \right) \| \nw_0-\nw_{\star} \|^2_{\nA}.
    \end{align*}

Observing that $\sum_{i=0}^{t}{\gamma^{i}(1-\lambda_\ell(\eproj))^{t-i}} = \gamma^{t}\sum_{i=0}^{t}{\alpha^{t-i}}$ where $\alpha \coloneqq \frac{1-\lambda_\ell(\eproj)}{\gamma}<1$, we have
\begin{align*}
     \E\|\nA^{1/2} (\nw_{t+1}-\nw_{\star}) \|^2_{Q_\ell} & \leq  (1 - \lambda_\ell(\eproj))^{t+1} \| \nA^{1/2} (\nw_0-\nw_{\star}) \|^2_{Q_\ell} + \frac{\lambda_{\ell+1}(\eproj (\nI-\eproj))}{1 - \alpha} \gamma^t \| \nw_0-\nw_{\star} \|^2_{\nA},
\end{align*}
which concludes the proof.
\end{proof}

We now use the tail averaging idea similar to \cite{epperly2024randomized}, obtaining fast convergence along the \topl{} subspace. 
The proof of \cref{theorem:fast_convg_tark_2} is then derived from the following result, after combining it with fixed-size DPP sampling guarantees of \cref{lemma:fix_size_dpp}.

\begin{theorem}[Fast convergence along subspace for tail-averaged iterate] \label{theorem:convg_sub_tark}
    Let $\nS \in \R^{2b \times n}$ be a random subsampling matrix sampled from $2b$-DPP($\nA$). For any $t > 2$ define $\hat \nw_t = \frac{2}{t} \sum_{i \geq t/2}^{t - 1}{\nw_i}$ where $\nw_i$ are generated using update rule \eqref{eqn:rcd_update}. 
    Then 
    \begin{align*}
        \E \| \nA^{1/2} (\hat \nw_t-\nw_{\star}) \|_{Q_\ell}^2 
        & \leq (1-\lambda_\ell(\eproj))^{t/2} \|\nA^{1/2} (\nw_0-\nw_{\star}) \|_{Q_\ell}^2 \\
        & \quad + 2 \frac{\lambda_{\ell+1}(\eproj (\nI - \eproj))}{t (1 - \alpha)^2}(1 - \lambda_n(\eproj))^{t/2 - 1} \|\nw_0 -\nw_{\star}\|^2_{\nA},
    \end{align*}
    where $\alpha = \frac{1 - \lambda_\ell(\eproj)}{1 - \lambda_n(\eproj)}$.
\end{theorem}

\begin{proof}
Let $\hat \nw_t = \frac{2}{t} \sum_{i\geq t/2}^{t - 1}{\nw_i}$. 
Consider $\E \| \nA^{1/2} (\hat\nw_t -\nw_{\star}) \|_{Q_\ell}^2$, set $t_a = t / 2$, and let $\E_{\nw_r}$ denote the expectation conditioned on $\nw_r$.
Applying the tower rule yields
\begin{align*}
    \E \| \nA^{1/2} (\hat \nw_t -\nw_{\star}) \|_{Q_\ell}^2 &= \frac{1}{(t - t_a)^2} 
    \sum_{s = t_a}^{t - 1} \sum_{r = t_a}^s \E[(\nw_r - \nw_{\star})^T \nA^{1/2} Q_\ell \nA^{1/2} (\nw_s - \nw_{\star})]\\
    &= \frac{1}{(t - t_a)^2} \sum_{s = t_a}^{t - 1} \sum_{r = t_a}^s \E \left[ (\nw_r-\nw_{\star})^T \nA^{1/2} Q_\ell \E_{\nw_r}[ Q_\ell \nA^{1/2} (\nw_s - \nw_{\star})] \right].
\end{align*}
% where $\E_{\nw_r}$ means conditional expectation given $\nw_r$. 

For $r < s$, we have
\begin{align*}
    \E_{\nw_r} [ Q_\ell \nA^{1/2} (\nw_s-\nw_{\star}) ] &= \nQ_\ell (\nI - \eproj) \E_{\nw_r}[ \nA^{1/2} (\nw_{s-1} - \nw_{\star}) ]\\
    &= \nQ_\ell \nQ_\ell (\nI - \eproj) \E_{\nw_r}[ \nA^{1/2} (\nw_{s-1} - \nw_{\star}) ]\\
    &= \nQ_\ell (\nI - \eproj) \nQ_\ell \E_{\nw_r}[\nA^{1/2} (\nw_{s-1} - \nw_{\star})].\\
\end{align*}

Recursing on the above relation yields
\begin{align*}
   \E_{\nw_r}[Q_\ell \nA^{1/2} (\nw_m - \nw_{\star})] &= \nQ_\ell (\nI - \eproj)^{m - r} \nQ_\ell [ \nQ_\ell \nA^{1/2} (\nw_r - \nw_{\star})].
\end{align*}

Therefore,
\begin{align*}
    \E \| \nA^{1/2} (\hat \nw_t - \nw_{\star}) \|_{\nQ_\ell}^2 &= \frac{1}{ (t - t_a)^2} \sum_{s = t_a}^{t - 1} \sum_{r = t_a}^s \E \left[ (\nw_r - \nw_{\star})^T \nA^{1/2} \nQ_\ell (\nI - \eproj)^{s-r} \nQ_\ell \nA^{1/2} ( \nw_r -\nw_{\star}) \right]\\
    &\leq \frac{1}{(t - t_a)^2} \sum_{s = t_a}^{t - 1} \sum_{r = t_a}^s (1 - \lambda_\ell(\eproj))^{s -r} \E \| \nA^{1/2} (\nw_r - \nw_{\star}) \|^2_{Q_\ell}.
\end{align*} 

Applying \cref{lemma:convg_l2} to upper bound $\E \| \nA^{1/2} (\nw_r - \nw_{\star}) \|^2_{Q_\ell}$ obtains
\begin{align*}
    \E \| \nA^{1/2} (\hat \nw_t - \nw_{\star}) \|_{Q_\ell}^2 
    & \leq \frac{1}{(t - t_a)^2} \| \nA^{1/2}(\nw_0 - \nw_{\star}) \|_{Q_\ell}^2 \sum_{s = t_a}^{t - 1} \sum_{r = t_a}^s (1 - \lambda_\ell(\eproj))^s \\
    & \quad + \frac{1}{(t - t_a)^2} \frac{\lambda_{\ell+1} (\eproj (\nI - \eproj))}{1 - \alpha} \| \nA^{1/2} (\nw_0 - \nw_{\star}) \|^2 \sum_{s = t_a}^{t - 1} \sum_{r = t_a}^s (1 - \lambda_\ell(\eproj))^{s - r} \gamma^{r - 1} \\
    & \leq (1 - \lambda_\ell(\eproj))^{t_a} \| \nA^{1/2} (\nw_0 - \nw_{\star}) \|_{Q_\ell}^2 \\
    & \quad + \gamma^{t_a - 1} \frac{\lambda_{\ell + 1}(\eproj (\nI - \eproj))}{1 - \alpha} \| \nA^{1/2} (\nw_0 - \nw_{\star}) \|^2 \frac{1}{(t - t_a)^2} \sum_{s = t_a}^{t - 1} \sum_{r = t_a}^s \alpha^{s - r},
\end{align*}
where $\gamma = 1 - \lambda_n(\eproj)$.

Using $\sum_{s = t_a}^{t - 1} \sum_{r = t_a}^s \alpha^{s - r} < (t - t_a)\sum_{i=0}^\infty \alpha^i = \frac{t -t_a}{1 - \alpha}$ we get,
\begin{align*}
     \E \| \nA^{1/2} (\hat \nw_t - \nw_{\star}) \|_{Q_\ell}^2 &\leq (1 - \lambda_\ell(\eproj))^{t_a} \| \nA^{1/2}(\nw_0 - \nw_{\star}) \|_{Q_\ell}^2 
     + \frac{\lambda_{\ell + 1}(\eproj (\nI - \eproj))}{(t - t_a)(1 - \alpha)^2} \gamma^{t_a - 1} \| \nA^{1/2} (\nw_0 - \nw_{\star}) \|^2.
\end{align*}

Substituting $t_a = t / 2$, the result immediately follows.
% \begin{align*}
%      \E \| \nA^{1/2} ( \hat \nw_t - \nw_{\star} ) \|_{Q_\ell}^2 &\leq (1 - \lambda_\ell(\eproj))^{t/2} \| \nA^{1/2}(\nw_0 - \nw_{\star}) \|_{Q_\ell}^2 
%      + 2 \frac{\lambda_{\ell + 1} (\eproj (\nI - \eproj) )}{t (1 - \alpha)^2} \gamma^{t/2 - 1} \| \nA^{1/2} (\nw_0 - \nw_{\star}) \|^2.
% \end{align*}
\end{proof}

\paragraph{Completing the proof of \cref{theorem:fast_convg_tark_2}}
It remains to use the guarantees for the eigenvalues of the matrix $\bar\Pi$ from Lemma \ref{lemma:fix_size_dpp}.
\begin{proof}[Proof of Theorem \ref{theorem:fast_convg_tark_2}]
Recalling the definition of $\alpha$ from \cref{theorem:convg_sub_tark} and our assumption that $\lambda_\ell(\eproj) \geq 2 \lambda_n(\eproj)$, we obtain for  
    $\gamma = 1- \lambda_n(\eproj)$ that $(1 - \alpha)^{-1} < 2 \gamma/\lambda_\ell(\eproj)$. 
    Consequently,
    \begin{align*}
        \frac{\lambda_{\ell + 1}(\eproj (\nI - \eproj))}{(1 - \alpha)^2} 
        &\leq  \frac{\lambda_{\ell}(\eproj (\nI - \eproj))}{(1 - \alpha)^2} \\
        &< \frac{4(1 - \lambda_\ell(\eproj))}{\lambda_\ell(\eproj)} \gamma^2 \\
        &< 4 \left( \frac{1}{\lambda_\ell(\eproj)} - 1 \right) \gamma\\
        &\overset{(1)}< 4 \left( \frac{1}{b} \sum_{i>b} \frac{\lambda_i}{\lambda_\ell} \right) \gamma\\
        &\overset{(2)}= 4 \phi(b, \ell) \gamma,
    \end{align*}
    where $(1)$ applies \cref{lemma:fix_size_dpp} and $(2)$ defines $\phi(b, \ell) = \frac{1}{b} \sum_{i > b} \frac{\lambda_i}{\lambda_\ell}$.
    Observing the elementary inequalities:
    \begin{equation*}
        \gamma < 1 - \frac{1}{2 \phi(b, n)}, \quad 1 - \lambda_\ell(\eproj) \leq \left(1 - \frac{1}{1 + \phi(b, \ell)} \right), 
    \end{equation*}
    we immediately deduce from \cref{theorem:convg_sub_tark} that
    \begin{align*}
        \E \| \nA^{1/2} ( \hat \nw_t - \nw_{\star} ) \|_{Q_\ell}^2 & \leq \left( 1 - \frac{1}{1 + \phi(b, \ell)} \right)^{t/2} \| \nA^{1/2}(\nw_0 - \nw_{\star}) \|^2_{\nQ_\ell} \\
        & \quad + \frac{8}{t} \phi(b, \ell) \left(1 - \frac{1}{2 \phi(b, n)} \right)^{t/2} \| \nw_0 - \nw_{\star} \|_{\nA}^2.
    \end{align*}
\end{proof}

\subsection{Posterior mean inference along \topl{} subspace for GPs}\label{s:proof-gp}
We begin by providing some background on GP inference in the Hilbert space setting. 
This allows for graceful transition from Hilbert space norm over the posterior mean to vector norms over $\R^{n}$. 
We recall that $f$ is a Gaussian process and $\{(x_i,y_i)\}_{i=1}^{n}$ represents the training data. 
The posterior Gaussian process is characterized by $\Nc(m_n(\cdot),  k_n(\cdot, \cdot))$, where
\begin{align*}
&m_n(\cdot) = m(\cdot) + k(\cdot, X)(K + \lambda I)^{-1} y, \\
    & k_n(\cdot, \cdot) = k(\cdot,\cdot) - k(\cdot, X)(K + \lambda I)^{-1}k(X, \cdot).
\end{align*}
Let $\Hc$ be the reproducing kernel Hilbert space (RKHS) associated with the kernel $k(\cdot,\cdot)$. 
Assuming $m(\cdot)=0$, the mean function $m_n(\cdot)$ can be identified as an element of the subspace $\Hc_n$ defined as
\begin{align*}
    \Hc_n:=\left\{\sum_{i=1}^{n}{w_i k(\cdot,x_i)} \ | \ w\in \R^n\right\}.
\end{align*}
In particular, $m_n = \sum_{i=1}^{n}{(w_\star)_{i} k(\cdot,x_i)}$, where $w_{\star}= (K+\lambda I)^{-1}\ny$. 
Furthermore, note that for any element $m' \in \Hc_n$, we have $\|m'\|_{\Hc}^2 = w^T Kw =\|w\|_{K}^2$. 
The operator $\Cc_n:=\frac{1}{n}\sum_{i=1}^{n}{k(\cdot,x_i)\otimes k(\cdot,x_i)}$ is known as the empirical covariance operator. 
Let $v_j$ denote the $j^{th}$ unit eigenvector of $K$ with eigenvalue $\lambda_j$. 
It is straightforward to show that $u_j:=\frac{1}{\sqrt{\lambda_j}}\sum_{i=1}^{n}{v_{ji} k(\cdot,x_i)}$ is a unit eigenvector of the unnormalized empirical covariance operator $\sum_{i=1}^{n}{k(\cdot,x_i) \otimes k(\cdot,x_i)}$ with eigenvalue $\lambda_j$. 
Let $V_\ell \in \R^{n\times \ell}$ consists of \topl{} orthogonal eigenvectors of $K$ as columns and $Q_\ell = V_\ell V_\ell^T$ be a projection matrix onto the subspace spanned by $v_1, \ldots, v_\ell$. 
Consider the $\ell$-dimensional subspace $\Hc_\ell$ defined as
\begin{align*}
\Hc_\ell:=\left\{\sum_{i=1}^{n}{(Q_\ell w)_{i} k(\cdot,x_i)} \ | \ w\in \R^n\right\}.
\end{align*}
We claim that $\Hc_\ell$ is the subspace formed by \topl{} eigenvectors of the empirical covariance operator. 
This can be seen clearly by choosing $w=v_j$ for $1\leq j\leq \ell$, we get $u_j \in \Hc_\ell$. 
We have the following conclusions.
\begin{itemize}
    \item For any $m' =\sum_{i=1}^{n}{w_i k(\cdot,x_i)}$, we have $\|m'-m_{n}\|_{\Hc}^2 = \|w-w_{\star}\|^2_{K}$.
    \item The element $m'_{Q_\ell} := \sum_{i=1}^{n}{(Q_\ell w)_{i} k(\cdot,x_i)}$ is an orthogonal projection of $m'$ onto $\Hc_\ell$. This can be seen as
    \begin{align*}
        m' =\sum_{i=1}^{n}{w_i k(\cdot,x_i)} = \underbrace{\sum_{i=1}^{n}{\left(Q_\ell w\right)_i k(\cdot,x_i)}}_{m'_{Q_\ell}} 
        + \underbrace{\sum_{i=1}^{n}{\left((I-Q_\ell )w\right)_i k(\cdot,x_i)}}_{m'_{Q^c_\ell}}.
    \end{align*}
    and finally noting that $\langle m'_{Q_\ell}, m'_{Q_\ell^c} \rangle_{\Hc} = w^T(I-Q_\ell)KQ_\ell w=0$.
    \item For any $m' =\sum_{i=1}^{n}{w_i k(\cdot,x_i)}$, we have $\| \rmproj_\ell(m') - \rmproj_\ell(m_{n})\|_{\Hc}^2 = \|Q_\ell(w-w_{\star})\|^2_{K}$, where $ \rmproj_\ell(m')$ denotes orthogonal projection of $m'$ onto $\Hc_\ell$.
\end{itemize}
We now derive the main result of this section by using Theorem \ref{theorem:fast_convg_tark_2}. 
We replace $\nA$ by $\nK_\lambda$ in the statement of Theorem \ref{theorem:fast_convg_tark_2} and use $\lambda_i'$ to denote the $i^{th}$ eigenvalue of $\nK_\lambda$, 
i.e., $\lambda_i+\lambda$ where $\lambda_i$ is the $i$th eigenvalue of $\nK$. 
Furthermore, Theorem \ref{theorem:main} can be derived from the following result by noticing that  sampling from $2b$-DPP($\nK_\lambda$) costs time $O(nb^2\log^4n)$ for preprocessing and an additional $O(b^3\log^3n)$ for actual sampling at every iteration (see Lemma \ref{lemma:dpp_sampling}). 
Here is the main result of the section:

\begin{theorem}[Subspace convergence for GP inference]\label{t:fast_convg_gp}
Let $b<n/2$ and $\nS \in \R^{2b\times n}$ be a random subsampling matrix sampled from $2b$-DPP($\nK_\lambda$). 
For any $t>2$ define $\hat\nw_t = \frac{2}{t}\sum_{i=t/2}^{t-1}{\nw_i}$ where $\nw_i$ are generated using the update rule \eqref{eqn:rcd_update}. Then, sketch-and-project initialized at $0$ satisfies
\begin{align*}
    \E\| \rmproj_\ell(\hat m_t) - \rmproj_\ell(m_n)\|_\Hc^2 
    &\leq\min\left\{\frac{8\phi(b,\ell)}{t}, \left(1-\frac{1}{2\phi(b,n)}\right)^{t/2}\right\}\|\ny\|_{\nK_\lambda^{-1}}^2.
\end{align*}
 where $\hat m_t = \sum_{i=1}^{n}{\hat\nw_{ti}k(\cdot,\nw_i)}$ and
 $\phi(b,p)= \frac{1}{b}\sum_{i>b}\frac{\lambda_i+\lambda}{\lambda_p+\lambda}$.
 \end{theorem}
\begin{proof}
If $\lambda_\ell(\bar\Pi) \geq 2\lambda_n(\bar\Pi)$, then we rely on Theorem \ref{theorem:fast_convg_tark_2} and replace $\nA$ by $\nK_\lambda$. We have the following observation: 
% If $\frac{1}{\phi(b,\ell)} \leq 1$ then for any $t>2$ we have
For any $t>2$, we have    
    \begin{align*}
        \left(1-\frac{1}{1+\phi(b,\ell)}\right)^{t/2} = \left(\frac{\phi(b,\ell)}{1+\phi(b,\ell)}\right)^{t/2} \leq \frac{2\phi(b,\ell)}{t}.
    \end{align*}
    % On the other hand if $\frac{1}{\phi(b,\ell)} > 1$, then for $t$ satisfying $t>4\log\left(\frac{t}{\phi(b,\ell)}\right)$, we get
    % \begin{align*}
    %     \left(1-\frac{1}{1+\phi(b,\ell)}\right)^{t/2} \leq \big(\phi(b,\ell)\big)^{t/2} \leq\frac{2\phi(b,\ell)}{t}.
    % \end{align*}
    Therefore for all $t>2$, 
    %for $t$ satisfying; $t > \max\{2, 4\log\left(\frac{t}{\phi(b,\ell)}\right)\}$, 
    we have,
    \begin{align*}
        \E\|\nK_{\lambda}^{1/2}(\hat\nw_t-\nw_{\star})\|^2_{Q_\ell} \leq \min\left\{\frac{8\phi(b,\ell)}{t}, \left(1-\frac{1}{2\phi(b,n)}\right)^{t/2}\right\}\|\nK_{\lambda}^{1/2}\nw_{\star}\|^2.
    \end{align*}
Let $\hat m_t = \sum_{i=1}^{n}{\hat \nw_{ti}k(\cdot,x_i)}$ where $\hat \nw_{ti}$ denote $i^{th}$ coordinate of $\hat\nw_t$ and $ m_n = \sum_{i=1}^{n}{\nw_{*i}k(\cdot,x_i)}$ where $\nw_{\star} = \nK_\lambda^{-1}\ny$. Then we have,

\begin{align*}
    \E\| \rmproj_\ell(\hat m_t) - \rmproj_\ell( m_n)\|_\Hc^2 &= \E\|Q_\ell(\hat\nw_t-\nw_{\star})\|^2_{\nK}\\
    &\leq \E\|Q_\ell(\hat\nw_t-\nw_{\star})\|^2_{\nK_\lambda}\\
    &\leq \min\left\{\frac{8\phi(b,\ell)}{t}, \left(1-\frac{1}{2\phi(b,n)}\right)^{t/2}\right\}\|\nK_\lambda^{1/2}\nw_{\star}\|^2\\
    &=\min\left\{\frac{8\phi(b,\ell)}{t}, \left(1-\frac{1}{2\phi(b,n)}\right)^{t/2}\right\}\|\ny\|_{\nK_\lambda^{-1}}^2.
\end{align*}
On the other hand if $\lambda_\ell(\bar\Pi) < 2\lambda_n(\bar\Pi)$, then we simply use the SAP analysis and get
\begin{align*}
    \E\|\nQ_\ell(\hat\nw_t-\nw_*)\|^2_{\nK_\lambda} \leq  \E\|\hat\nw_t-\nw_*\|^2_{\nK_\lambda}&= \frac{4}{t^2}\cdot\E \left \|\sum_{i=t/2}^{t-1}{(\nw_i-\nw_*)}\right \|_{\nK_\lambda}^{1/2}\\
    &=\frac{4}{t^2} \left( \sum_{i=t/2}^{t-1}{\E\|\nw_i-\nw_*\|_{\nK_\lambda}^2} + \sum_{s,r=t/2,s>r}^{t-1}{\E(\nw_s-\nw_*)\nK_\lambda(\nw_r-\nw_*)} \right)\\
    &\leq \left( 1-\frac{1}{2\phi(b,n)} \right)^{t/2}\|\nK_\lambda^{1/2}\nw_*\|^2,
\end{align*}
where the last inequality can be obtained using a recursive argument similar to Theorem \ref{theorem:convg_sub_tark} by plugging in linear convergence guarantees for SAP (using Lemma \ref{lemma:cord_desc} and Lemma \ref{lemma:convg_subsp_1} with $\ell=n$). We get,
\begin{align*}
     \E\| \rmproj_\ell(\hat m_t) - \rmproj_\ell( m_n)\|_\Hc^2 \leq  \E\|Q_\ell(\hat\nw_t-\nw_{\star})\|^2_{\nK_\lambda} &\leq \left( 1-\frac{1}{2\phi(b,n)} \right)^{t/2}\|\nK_\lambda^{1/2}\nw_*\|^2\\
     &=\min \left\{ \frac{4\phi(b,n)}{t}, \left( 1-\frac{1}{2\phi(b,n)} \right)^{t/2} \right\}\|\ny\|^2_{\nK_\lambda^{-1}}\\
     &< \min \left\{ \frac{8\phi(b,\ell)}{t},\left( 1-\frac{1}{2\phi(b,n)} \right)^{t/2} \right\}\|\ny\|^2_{\nK_\lambda^{-1}}.
\end{align*}
Combining the results for both scenarios: $\lambda_\ell(\bar\Pi) \geq 2\lambda_n(\bar\Pi)$ or $\lambda_\ell(\bar\Pi) < 2\lambda_n(\bar\Pi)$, we conclude the proof.
% Finally, considering $\nS$ to be a subsampling matrix drawn approximately from $2b$-DPP of $\nK_\lambda$ and taking the union bound with high probability guarantees in Lemma \ref{lemma:dpp_sampling} concludes the proof.
\end{proof}

\paragraph{Time complexity analysis}
We now use the above guarantee to provide the time complexity analysis for estimating the GP posterior mean. The following result immediately implies  Corollary \ref{corollary:main}.
\begin{corollary}\label{c1:fast_convg_gp}
    Suppose that the matrix $K$ exhibits polynomial spectral decay, i.e., $\lambda_i(K) = \Theta(i^{-\beta})$ for some $\beta > 1$. 
    Then for any $\ell \in \{1,\ldots,n\}$, $\lambda=O(1)$ and $\epsilon\in(0,1)$, choosing $b=2\ell$ we can find $\hat m$ that with probability at least $0.99$ satisfies $\|\mathrm{proj}_\ell(\hat m) - \mathrm{proj}_\ell(m_n)\|^2_{\Hc} \leq \epsilon  \|\mathrm{proj}_\ell(m_n)\|^2_{\Hc}$ in
    \begin{align*}
    \bigO\left( n\ell^2\log^4n + (n^2 + n\ell^2\log^3n) \min \left\{ \frac{\log(n/\ell)}{\epsilon}, \left( 1 + \frac{\ell(\lambda_\ell(K)+\lambda)}{n(\lambda_n(K)+\lambda)} \right)\log(n/\ell\epsilon)  \right\} \right) \quad \text{time}.
    \end{align*}
 \end{corollary}

\begin{proof}
As $\ny = f(X) + g$, where $ g \sim \Nc(0, \lambda \nI)$ and $f( X)\sim \Nc(0,\nK)$, we have $\ny \sim \Nc(0,\nK_\lambda)$. 
This implies $\E[\ny\ny^T]=\nK_\lambda$. 
Therefore, $\E\|\nK_{\lambda}^{1/2}\nw_{\star}\|^2=n$ and $\E\|Q_\ell\nK_{\lambda}^{1/2}\nw_{\star}\|^2=\ell$. 
So, using standard Gaussian concentration of measure, it follows that with probability $0.999$, $\|\nK_{\lambda}^{1/2}\nw_{\star}\|^2 \leq \bigO(1)\frac{n}{\ell}\|\nK_{\lambda}^{1/2}\nw_{\star}\|^2_{Q_\ell}$. 
Furthermore, 
\[
\|\nK_\lambda^{1/2}\nw_{\star}\|^2_{Q_\ell} \leq \left(1+\frac{\lambda}{\lambda_\ell}\right)\|\nK^{1/2}\nw_{\star}\|_{Q_\ell}^2 
= \left(1+\frac{\lambda}{\lambda_\ell}\right)\| \rmproj_\ell( m_n)\|^2_\Hc.
\] 
We get the following relation:

\begin{align*}
    \E\| \rmproj_\ell(\hat m_t) - \rmproj_\ell( m_n)\|_\Hc^2 \leq \bigO(1)\left[\frac{n(1+\lambda/\lambda_\ell)}{\ell}\min\left\{\frac{\phi(b,\ell)}{t}, \left(1-\frac{1}{4\phi(b,n)}\right)^{t/2}\right\}\right] \| \rmproj_\ell( m_n)\|_\Hc^2
\end{align*}

Now let $b>\ell +\sum_{i>b}{\lambda_i/{(\lambda_\ell+\lambda)}}$. 
We have $\phi(b,\ell) \leq \frac{1}{b}\sum_{i>b}\frac{\lambda_i}{\lambda_\ell+\lambda} + \frac{n\lambda}{b(\lambda_\ell+\lambda)} \leq 1+ \frac{n\lambda}{b(\lambda_\ell+\lambda)}$. We consider following two cases: 
\vspace{1mm}

\noindent \emph{Case 1: $\frac{n\lambda}{b(\lambda_\ell+\lambda)} \leq 1$.} In this case we get $\phi(b,\ell) <2$. After $t=O(\frac{n}{\ell\epsilon})$ iterations we get
$\E\| \rmproj_\ell(\hat m_t) - \rmproj_\ell(m_n)\|^2_\Hc \leq \epsilon\| \rmproj_\ell( m_n)\|^2_\Hc.$

\vspace{1mm}

\noindent \emph{Case 2: $\frac{n\lambda}{b(\lambda_\ell+\lambda)} > 1$. }
We get $\lambda > \frac{b\lambda_\ell}{n}$. In this case we have $\phi(b,n) = \frac{1}{b}\sum_{i>b}\frac{\lambda_i+\lambda}{\lambda_n+\lambda} < \frac{n}{b}+ \frac{1}{b}\sum_{i>b}\frac{\lambda_i}{\lambda} < \frac{2n}{b}$. 
Therefore, after $t = \bigO \left( \frac{n \log \left( \frac{n (1 + \lambda / \lambda_\ell)}{\ell \epsilon} \right)}{b} \right)$ iterations, we get $\E\| \rmproj_\ell(\hat m_t) - \rmproj_\ell(m_n)\|^2_\Hc \leq \epsilon\| \rmproj_\ell(m_n)\|^2_\Hc.$
\vspace{1mm}

\noindent On the other hand, in either case 
\begin{align*}
    \phi(b,n) & = \frac{1}{b}\sum_{i>b}\frac{\lambda_i+\lambda}{\lambda_n+\lambda} 
    % & < \frac{n}{b}\frac{1}{n-s}\sum_{i>b}\frac{\lambda_i+\lambda}{\lambda_n+\lambda}\\ 
    \leq \frac{n}{b}\left(1+\frac{(\lambda_\ell+\lambda)}{(\lambda_n+\lambda)}\frac{1}{n}\sum_{i>b}\frac{\lambda_i}{\lambda_\ell+\lambda}\right) 
     < \frac{n}{b}\left(1+\frac{b}{n}\frac{\lambda_\ell+\lambda}{\lambda_n+\lambda}\right), 
\end{align*}
as we assumed $b>\sum_{i>b}\frac{\lambda_i}{\lambda_\ell+\lambda}$.
Furthermore, for given spectral decay for any $\beta>1$, we have $\ell + \sum_{i>b}\frac{\lambda_i}{\lambda_\ell} < 2\ell$. 
This implies for $b = 2 \ell$ we obtain $\epsilon$ accuracy result in 
$t=\bigO\left(\frac{n}{\ell}\min\left\{\frac{\log(n/\ell)}{\epsilon}, \left(1+\frac{\ell}{n}\frac{\lambda_\ell+\lambda}{\lambda_n+\lambda}\right)\log(n/\ell\epsilon)\right\}\right)$, where we used $\lambda = \bigO(1)$ to get $\log \left( \frac{n}{\ell}(1+\lambda/\lambda_\ell) \right) = \bigO \left( \log(\frac{n}{\ell}) \right)$. 
Then, after applying Markov's inequality, $\hat m= \hat m_t$ obtains the desired $\epsilon$ accuracy with probability 0.999. 
The total time complexity follows by combining the number of iterations with cost of approximate sampling from $2b$-DPP($\nA$) (\Cref{lemma:dpp_sampling}) and additional per iteration cost of $\bigO(n \ell + \ell^3 \log^3n)$. 
Note that while the sampling algorithm is not exact, 
taking the union bound with respect to the high probability guarantees in \Cref{lemma:dpp_sampling} ensures that all of the samples are indistinguishable from true DPP with probability 0.999. Finally, union bounding over the three 0.999 probability events we have invoked concludes the proof. 
\end{proof}

\paragraph{Improved guarantee for very ill-conditioned problems} Here, we show that when $\lambda$ is very small (i.e., $K_\lambda$ is highly ill-conditioned), then we can obtain an even better time complexity in the first phase of the convergence, addressing the claim in Remark \ref{remark:main}.
 \begin{corollary}\label{corollary:fast_convg_gp} 
 Suppose that the matrix $K$ exhibits polynomial spectral decay, i.e., $\lambda_i(K) = \Theta(i^{-\beta})$ for some $\beta > 1$ and $\lambda<\frac{1}{nb^{\beta-1}}$, 
 then we can find $\hat m$ that with probability at least $0.99$ satisfies $\|\mathrm{proj}_\ell(\hat m) - \mathrm{proj}_\ell(m_n)\|_{\Hc}^2 \leq \epsilon  \|\mathrm{proj}_\ell(m_n)\|_{\Hc}^2$ in
 \begin{align*}
       \bigO\left(nb^2\log^4n + (n^2+nb^2\log^3n)(\ell/b)^{\beta-1}/\epsilon\right).
   \end{align*}
\end{corollary}

\begin{proof}
   We reconsider case 1 from the proof of \cref{c1:fast_convg_gp}. 
   Using the given spectral decay profile for $\nK$ and that $\lambda < \frac{1}{nb^{\beta-1}}$, 
   we get $ \phi(b,\ell) = \bigO\left(\frac{b^{-\beta}}{\ell^{-\beta}}\right)$. 
   Therefore after $t=\frac{n}{\ell\epsilon}\frac{b^{-\beta}}{\ell^{-\beta}}$ we obtain the $\epsilon$ approximation guarantee. 
   The total runtime complexity becomes:
   \begin{align*}
       \bigO\left(nb^2\log^4n + (n^2+nb^2\log^3n)(\ell/b)^{\beta-1}/\epsilon\right).
   \end{align*}

\end{proof}

\section{Additional Algorithmic Details for \adasap{}}
\label{sec:add_alg_deets}
We provide additional implementation details for \adasap{}.
\cref{subsec:dist_mat_mat} describes how we distribute matrix-matrix products across rows and columns in the algorithms \algcoldist{} and \algrowdist{}.
\cref{subsec:nystrom} describes the practical implementation of the randomized \nys{} approximation and provides pseudocode for \algnys{}.
\cref{subsec:get_l} describes how we compute preconditioned smoothness constants and provides pseuodocode for \algstpsz{}.
\cref{subsec:nest_acc_appdx} provides pseudocode for Nesterov acceleration in \algnestacc{}.
\cref{subsec:adasap_parallel_scaling} provides a scaling plot illustrating the speedups achieved by using multiple GPUs in \adasap{}.
\cref{subsec:tail_avg_appdx} investigates the impact of tail averaging on the performance of \adasap{}.

All operations involving kernel matrices are performed using pykeops \citep{charlier2021kernel}, which allows us to avoid instantiating kernel matrices explicitly in memory.
To see the full details of our implementation, we recommend the reader to view our 
\ifarxivpreprint
    \href{\codeurlpublic}{codebase}.
\else
    \href{\codeurlanon}{codebase}.
\fi

\subsection{Distributed matrix-matrix products}
\label{subsec:dist_mat_mat}
Here, we provide details for how we implement the distributed matrix-matrix products in \adasap{}.
\algcoldist{} (\cref{alg:col_dist_mat_mat}) shows how we distribute the matrix-matrix product $\Kbn W_t$ in \adasap{} and \algrowdist{} (\cref{alg:row_dist_mat_mat}) shows how we distribute the calculation of the \nys{} sketch in \adasap{}.
Our implementation of these algorithms uses torch.multiprocessing to spawn a CUDA context on each device (i.e., a worker) and uses pykeops to generate the column and row block oracles.

\begin{algorithm}
    \caption{\algcoldist}
    \label{alg:col_dist_mat_mat}
    \begin{algorithmic}
    \Require Right-hand side matrix $W \in \R^{n \times \nrhs}$, row indices $\B$, workers $\{\W_1, \ldots, \W_{\nworks}\}$
    \State Partition rows of $W$ as $\{W_1, \ldots, W_\nworks\}$
    \State Send $W_i$ to $\W_i$
    \State Generate column block oracle $\Kcol_{\W_{i}}$ using $\B$
    \State $(\Kbn W)_i \gets \Kcol_{\W_{i}}[W_i]$ \Comment{Compute column block products in parallel}
    \State Aggregate $\Kbn W \gets \sum_{i = 1}^{\nworks} (\Kbn W)_i$
    \State \Return $\Kbn W$
    \end{algorithmic}
\end{algorithm}

\begin{algorithm}
    \caption{\algrowdist}
    \label{alg:row_dist_mat_mat}
    \begin{algorithmic}
    \Require Right-hand side matrix $\Omega \in \R^{n \times r}$, row indices $\B$, workers $\{\W_1, \ldots, \W_{\nworks}\}$
    \State Send $\Omega$ to each $\W_i$
    \State Generate row block oracle $\Krow_{\W_{i}}$ using $\B$
    \State $(\Kbb \Omega)_i \gets \Krow_{\W_{i}}[\Omega]$ \Comment{Compute row block products in parallel}
    \State Aggregate $\Kbb \Omega \gets \left[(\Kbb \Omega)_1^T~ \ldots ~ (\Kbb \Omega)_\nworks^T \right]^T$
    \State \Return $\Kbb \Omega$
    \end{algorithmic}
\end{algorithm}

\subsection{Randomized \nys{} approximation}
\label{subsec:nystrom}
Here, we present a practical implementation of the \nys{} approximation used in \adasap{} (\cref{alg:adasap}) in \algnys{} (\cref{alg:nystrom}).
The Randomized Nystr{\"o}m approximation of $\Kbb$ with test matrix $\Omega \in \R^{b\times r}$ \citep{tropp2017fixed} is given by:
\[
\Knys = (\Kbb \Omega)(\Omega^{T}\Kbb \Omega)^{\dagger}(\Kbb\Omega)^{T}.
\]
The preceding formula is numerically unstable, so \adasap{} uses \algnys{}, which is based on \citet[Algorithm 3]{tropp2017fixed}. 
$\text{eps}(x)$ is defined as the positive distance between $x$ and the next largest floating point number of the same precision as $x$. 
The resulting Nystr\"{o}m approximation $\hat{M}$ is given by $USU^T$, where $U \in \R^{p \times r}$ is an orthogonal matrix that contains the approximate top-$r$ eigenvectors of $M$, and $S \in \R^r$ contains the top-$r$ eigenvalues of $M$.
The \nys{} approximation is positive-semidefinite but may have eigenvalues that are equal to $0$.
In our algorithms, this approximation is always used in conjunction with a regularizer to ensure positive definiteness.

\begin{algorithm}
    \caption{\algnys{}}
    \label{alg:nystrom}
    \begin{algorithmic}
        \Require sketch $M\Omega \in \R^{p\times r}$, sketching matrix $\Omega \in \R^{p\times r}$, approximation rank $r \leq p$
        % \State $\Omega \gets \mathrm{randn}(p, r)$
        % \hfill \Comment{Test matrix}
        % \State $\Omega \gets \mathrm{thin\_qr}(\Omega )$
        % \hfill \Comment{Orthogonalize test matrix}
        \State $\Delta \gets \mathrm{eps}(\Omega^{T}M\Omega.\mathrm{dtype}) \cdot \Tr(\Omega^{T}M\Omega)$ 
        \hfill \Comment{Compute shift for stability}
        % \hfill \Comment{Compute sketch, adding shift for stability}
        \State $C \gets \text{chol}(\Omega^TM\Omega + \Delta \Omega^T \Omega)$ 
        \hfill \Comment{Cholesky decomposition: $C^{T}C = \Omega^{T}M\Omega + \Delta \Omega^T \Omega$}
        \State $B \gets YC^{-1}$ 
        \hfill \Comment{Triangular solve}
        \State $[U, \Sigma, \sim] \gets \text{svd}(B, 0)$ 
        \hfill \Comment{Thin SVD}
        \State $S \gets \text{max}\{0, \diag(S^2 - \Delta I)\}$ 
        \hfill \Comment{Compute eigs, and remove shift with element-wise max}
        \State \Return $U, S$
    \end{algorithmic}
\end{algorithm}

The dominant cost in \cref{alg:nystrom} is computing the SVD of $B$ at a cost of $\bigO(p r^2)$.
This is the source of the $\bigO(pr^2)$ cost in Phase III of \adasap{}. 

\subsubsection{Applying the \nys{} approximation to a vector}
In our algorithms, we 
often perform computations of the form $(\hat{M} + \rho I)^{-1} g = (U S U^T + \rho I)^{-1} g$.
This computation can be performed in $\bigO(rp)$ time using the Woodbury formula \citep{higham2002accuracy}:
\begin{align}
\label{eq:inv_nys_woodbury}
        (U S U^{T} + \rho I)^{-1} g
        &= U \left(S + \rho I \right)^{-1} U^{T} g + \frac{1}{\rho} (g - U U^{T} g).
\end{align}

We also use the randomized \nys{} approximation to compute preconditioned smoothness constants in \algstpsz{} (\cref{alg:get_stpsz}).
This computation requires the calculation $(P + \rho I)^{-1/2} v$ for some $v \in \R^p$, which can also be performed in $\bigO(pr)$ time using the Woodbury formula:
\begin{align}
\label{eq:inv_sqrt_nys_woodbury}
        (U S U^T + \rho I)^{-1/2} v = U \left(S + \rho I \right)^{-1/2} U^{T} v + \frac{1}{\sqrt{\rho}} (v - U U^{T} v).
\end{align}

In single precision, \eqref{eq:inv_nys_woodbury} is unreliable for computing $(P + \rho I)^{-1} g$. 
This instability arises due to roundoff error: 
the derivation of \eqref{eq:inv_nys_woodbury} assumes that $\hat U^T \hat U = I$, but we have empirically observed that orthogonality does not hold in single precision.
To improve stability, we compute a Cholesky decomposition $LL^T$ of $\rho S^{-1} + U^T U$, which takes $\bigO(pr^2)$ time since we form $U^T U$.
Using the Woodbury formula and Cholesky factors, 
\begin{align*}
   (U S U^T + \rho I)^{-1} g &= \frac{1}{\rho} g - \frac{1}{\rho} 
   U (\rho S^{-1} + U^T U)^{-1} U^T g  \\
   &= \frac{1}{\rho} g - \frac{1}{\rho} U L^{-T} L^{-1} U^T g.
\end{align*}
This computation can be performed in $\bigO(pr)$ time, since the $\bigO(r^2)$ cost of triangular solves with $L^T$ and $L$ is negligible compared to the $\bigO(pr)$ cost of multiplication with $U^T$ and $U$.

Unlike \cref{eq:inv_nys_woodbury}, we find using \eqref{eq:inv_sqrt_nys_woodbury} in \algstpsz{} yields excellent performance, i.e., we do not need to perform any additional stabilization.
% This is consistent with the findings in \cite{rathore2025have}.

\subsection{Computing the stepsize}
\label{subsec:get_l}
Here, we provide the details of the \algstpsz{} procedure in \adasap{} (\cref{alg:adasap}) for automatically computing the stepsize.
Our procedure is inspired by \cite{rathore2025have}, who show that approximate \sap{} with the Nystr{\"o}m approximation converges when
\[
\eta_\B = 1/\lambda_1\left((P_\B+\rho I)^{-1/2}(\Kbb + \lambda I)(P_\B+\rho I)^{-1/2}\right).
\]
That is, the correct stepsize to use is the reciprocal of the ``preconditioned subspace smoothness constant''.

To compute $\lambda_1\left((P_\B+\rho I)^{-1/2}(\Kbb + \lambda I)(P_\B+\rho I)^{-1/2}\right)$, \algstpsz{} uses randomized powering \citep{kuczynski1992estimating}.
This technique has been used in several previous works on preconditioned optimization to great effect \citep{frangella2024promise, frangella2024sketchysgd, rathore2025have}.
% To determine $\eta_B$, \adasap{} must compute $\lambda_1\left((P_\B+\rho I)^{-1/2}(\Kbb + \lambda I)(P_\B+\rho I)^{-1/2}\right)$.
% \adasap{} adopts a commonly used strategy in the literature in preconditioned optimization \citep{frangella2024promise, frangella2024sketchysgd, rathore2025have} to compute this quantity.
% Namely, it uses randomized powering \citep{kuczynski1992estimating}.
Given a symmetric matrix $H$, preconditioner $P$, and damping $\rho$,
randomized powering computes  
\[
\lambda_1((P + \rho I)^{-1/2} H (P + \rho I)^{-1/2}),
\]
using only matrix-vector products with the matrices $H$ and $(P + \rho I)^{-1/2}$. 
When $P$ is calculated using \algnys{}, \algstpsz{} can efficiently compute a matrix-vector product with $(P + \rho I)^{-1/2}$ using \eqref{eq:inv_sqrt_nys_woodbury}.
In practice, we find that 10 iterations of randomized powering are sufficient for estimating the preconditioned smoothness constant.
The pseudocode for randomized powering, based on the presentation in \cite{martinsson2020randomized}, is shown in \cref{alg:get_stpsz}.
Since \adasap{} only runs \algstpsz{} for $10$ iterations, the total cost of the procedure is $\bigO(b^2)$, which makes it the source of the $\bigO(b^2)$ cost of Phase III in \cref{alg:adasap}.

\begin{algorithm}
    \caption{\algstpsz{}}
    \label{alg:get_stpsz}
    \begin{algorithmic}
        \Require symmetric matrix $H$, preconditioner $P$, damping $\rho$, maximum iterations $N \gets 10$
        \State $v_0 \gets \mathrm{randn}(P.\mathrm{shape}[0])$
        \State $v_0 \gets v_0 / \|v_0\|_2$ 
        \hfill \Comment{Normalize}
        \For{$i = 0, 1, \ldots, N - 1$}
        \State $v_{i + 1} \gets (P + \rho I)^{-1/2} v_i$
        \State $v_{i + 1} \gets H v_{i + 1}$
        \State $v_{i + 1} \gets (P + \rho I)^{-1/2} v_{i + 1}$
        \State $v_{i + 1} \gets v_{i + 1} / \|v_{i + 1}\|_2$ 
        \hfill \Comment{Normalize}
        \EndFor
        \State $\lambda \gets (v_{N - 1})^T v_N$
        \State \Return $1/\lambda$
    \end{algorithmic}
\end{algorithm}

\subsection{Nesterov acceleration}
\label{subsec:nest_acc_appdx}
We present pseudocode for Nesterov acceleration in \algnestacc{} (\cref{alg:nest_acc}).

\begin{algorithm}
    \caption{\algnestacc}
    \label{alg:nest_acc}
    \begin{algorithmic}
        \Require iterates $\iterw{t}{}$, $\iterv{t}{}$, $\iterz{t}{}$, search direction $\iterd{t}{}$, stepsize $\eta_\B$, acceleration parameters $\beta$, $\gamma$, $\alpha$
        \State $\iterw{t + 1}{} \gets \iterz{t}{} - \eta_\B \iterd{t}{}$
        \State $\iterv{t + 1}{} \gets \beta \iterv{t}{} + (1 - \beta) \iterz{t}{} - \gamma \eta_\B \iterd{t}{}$
        \State $\iterz{t + 1}{} \gets \alpha \iterv{t}{} + (1 - \alpha) \iterw{t + 1}{}$
        \State \Return $\iterw{t + 1}{}$, $\iterv{t + 1}{}$, $\iterz{t + 1}{}$
    \end{algorithmic}
\end{algorithm}

\subsection{Parallel scaling of \adasap{}}
\label{subsec:adasap_parallel_scaling}
Here we present \cref{fig:parallel_scaling_taxi}, which shows the parallel scaling of \adasap{} on the taxi dataset.

\begin{figure}
    \centering
    \includegraphics[width=0.5\linewidth]{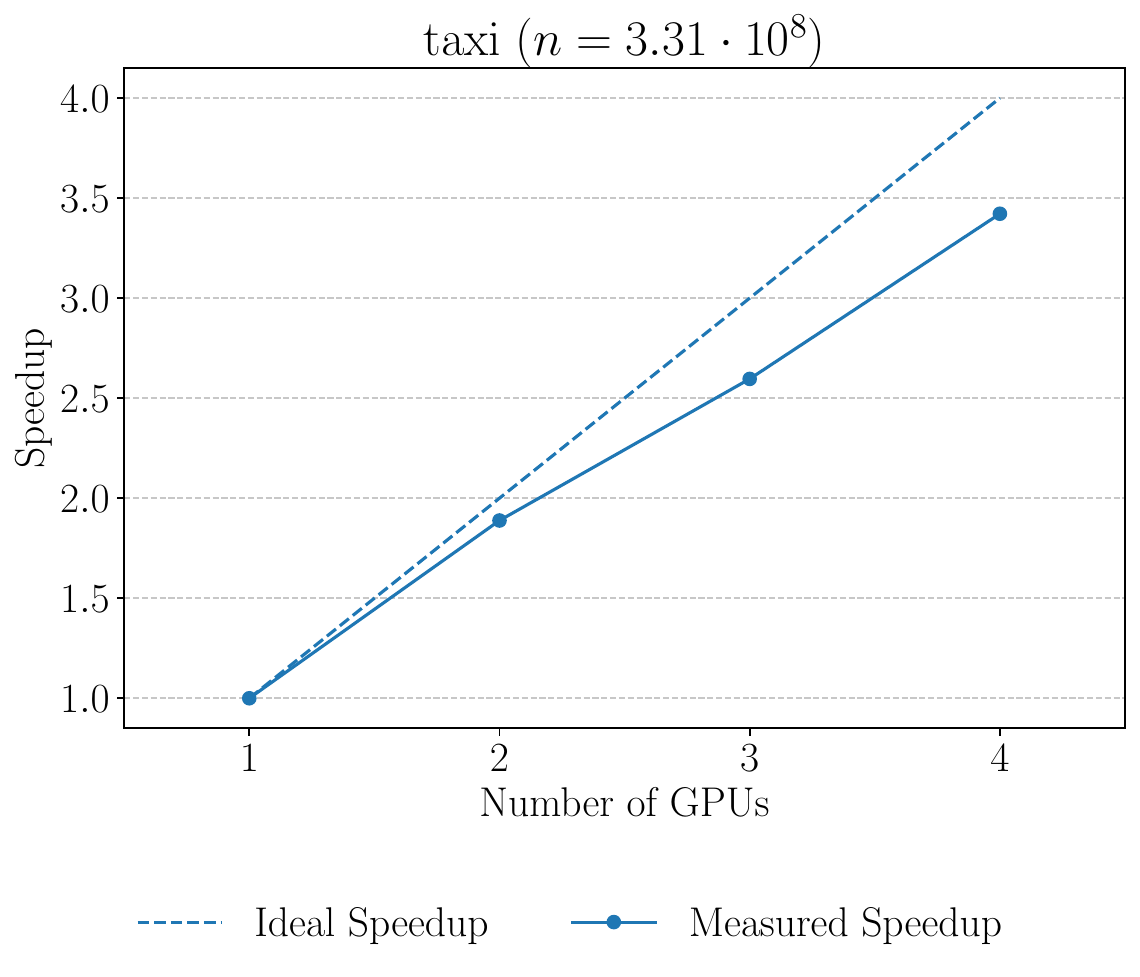}
    \caption{Multi-GPU scaling of \adasap{} on the taxi dataset. 
    \adasap{} obtains near-linear scaling with the number of GPUs.}
    \label{fig:parallel_scaling_taxi}
\end{figure}

\subsection{Impact of tail averaging on performance}
\label{subsec:tail_avg_appdx}
The theoretical results in \cref{sec:sap} require tail averaging.
Here we assess whether tail averaging leads to practical improvements in the convergence of sketch-and-project algorithms for kernel ridge regression.
To do so, we run a synthetic experiment using an RBF kernel with $n = 1000$ samples.

\cref{fig:subspace_convergence} displays the relative errors over the top-$\ell$ subspace, which are computed using the expression
\[
    \| \rmproj_\ell(\hat m) -  \rmproj_\ell(m_n)\|^2_{\Hc} / \| \rmproj_\ell(m_n)\|^2_{\Hc}.
\]
% \[
% \|W_t - W_\star\|_2^2 / \|W_\star\|_2^2
% \]
When $\ell = 10$ and the blocksize $b$ is small, tail averaging slightly improves the convergence rate.
However, when $\ell \in \{100, 1000\}$, tail averaging does not improve the convergence rate.
In fact, as the blocksize increases, tail averaging results in \textit{slower} convergence!

\begin{figure}
    \centering
    \includegraphics[width=\linewidth]{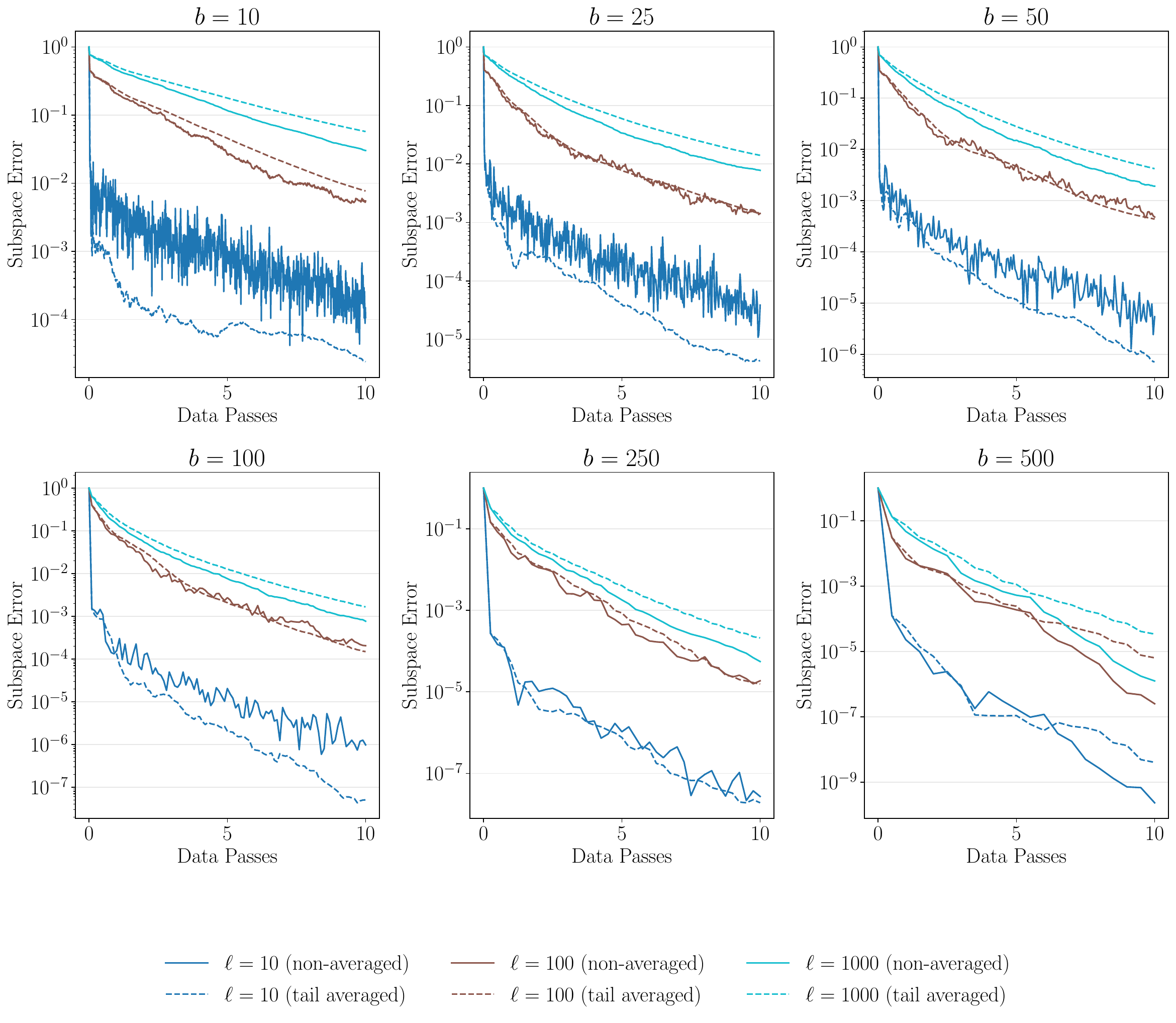}
    \caption{Performance of \adasap{} with and without tail averaging. 
    One ``data pass'' corresponds to one pass through the kernel matrix.
    Tail averaging does not improve convergence by a substantial margin.}
    \label{fig:subspace_convergence}
\end{figure}
\section{Additional Details for Experiments}
\label{sec:experiments_appdx}
Here we provide additional details for the experiments that are not provided in the main paper.

\subsection{Determining hyperparameters for regression}
We use a zero-mean prior for all datasets.
We train the kernel variance, likelihood variance, and lengthscale (we use a separate lengthscale for each dimension of $X$) using the procedure of \cite{lin2023sampling}, which we restate for completeness:
\begin{enumerate}
    \item Select a centroid point from the training data $X$ uniformly at random.
    \item Select the 10{,}000 points in the training data that are closest to the centroid in Euclidean norm.
    \item Find hyperparameters by maximizing the exact GP likelihood over this subset of training points.
    \item Repeat the previous three steps for 10 centroids and average the resulting hyperparameters.
\end{enumerate}

\subsection{Optimizer hyperparameters}
We present the hyperparameters for \adasap{}, \adasapi{}, \pcg{}, and \sdd{} that were not described in the main paper.

For GP inference on large-scale datasets, we use blocksize $b = n / 100$ in \adasap{}, \adasapi{}, and \sdd{}; blocksize $b = n / 2{,}000$ for transporation data analysis, and $b = n / 5$ for Bayesian optimization,

We set the rank $r = 100$ for both \adasap{} and \pcg{}.

Similar to \cite{lin2024stochastic}, we set the stepsize in \sdd{} to be one of $\{1/n, 10/n, 100/n\}$ (this grid corresponds to \sdd-1, \sdd-10, and \sdd-100), the momentum to $0.9$, and the averaging parameter to $100 / T_{\max}$.

\subsection{Additional details for GP inference experiments}
song and houseelec are from the UCI repository, yolanda and acsincome are from OpenML, and benzene and malonaledehyde are from sGDML \citep{chimela2017machine}.
We select the kernel function $k$ for each dataset based on previous work \citep{lin2023sampling,epperly2024embrace,rathore2025have}.

We standardize both the features and targets for each dataset.
For fairness, we run all methods for an equal amount of \textit{passes} through each dataset: we use 50 passes for yolanda, song, benzene, and malonaldehyde and 20 passes for acsincome and houseelec.

We use pathwise conditioning with 2{,}048 random features to (approximately) sample from the GP posterior.

\subsection{Additional details for transporation data analysis}
We standardize both the features and targets and use a RBF kernel.
Due to computational constraints, we use a single 99\%-1\% train-test split, and run each method for a single pass through the dataset.

\subsection{Additional details for Bayesian optimization}
Our implementation of Bayesian optimization largely mirrors that of \cite{lin2023sampling}.
We only present the high-level details here, and refer the reader to \cite{lin2023sampling} for the fine details of the implementation.

We draw the target functions $f: [0,1]^8 \rightarrow \R$ from a zero-mean GP prior with \mtrn-3/2 kernel using 5{,}000 random features.
At each iteration, we choose the acquisition points using parallel Thompson sampling \citep{hernandezlobato2017parallel}.
As part of this process, we use the multi-start gradient optimization maximization strategy given in \cite{lin2023sampling}.
At each step, we acquire 1{,}000 new points, which we use to evaluate the objective function.
Concretely, if $x_{\mathrm{new}}$ is an acquired point, we compute $y_{\mathrm{new}} = f(x_{\mathrm{new}}) + \epsilon$, where $\epsilon \sim \Nc(0, 10^{-6})$.
We then add $(x_{\mathrm{new}}, y_{\mathrm{new}})$ to the training data for the next step of optimization.
In our experiments, we initialize all methods with a dataset consisting of 250{,}000 observations sampled uniformly at random from $[0,1]^8$.

% All methods are warm-started with the same set of 250{,}000 points sampled uniformly at random from the domain.
% We then run 10 iterations of parallel Thompson sampling \citep{hernandezlobato2017parallel}---we acquire 1{,}000 points at each iteration, leading to 10{,}000 acquisitions in total.
% Due to the (comparatively) small size of the task, we run each method on 1 GPU.
% We also run a random search baseline, which acquires 10k additional points uniformly at random. 

\subsection{Additional timing plots for \cref{subsec:gp_inf}}
Here we present timing plots for the datasets used in \cref{subsec:gp_inf} that were not shown in the main paper.
\cref{fig:all_appendix_time} shows that all the methods (except \pcg{}) perform similarly on both test RMSE and test mean NLL.
However, \cref{fig:all_appendix_time_train_rmse} shows that \adasap{} and \pcg{} attain a much lower train RMSE than the competitors.
This suggests that the similar performance of the methods on test RMSE and test mean NLL is not due to differences in optimization, but rather, it is because the datasets are not well-modeled by Gaussian processes.

\begin{figure}
    \centering
    \includegraphics[width=\linewidth]{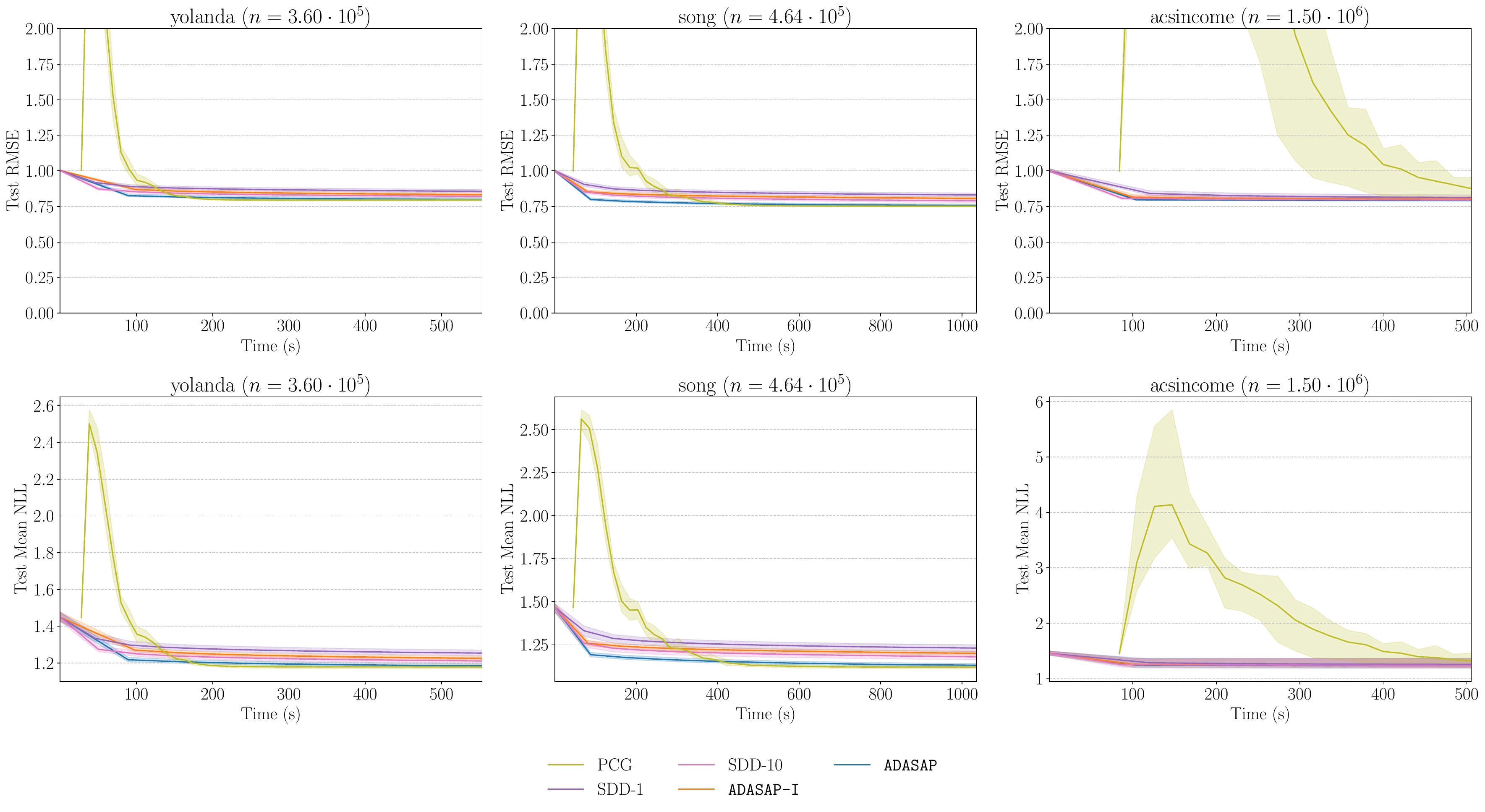}
    \caption{Performance of \adasap{} and competitors on RMSE and mean NLL, as a function of time, for benzene, malonaldehyde, and houseelec.
    The solid curve indicates mean performance over random splits of the data;
    the shaded regions indicate the range between the worst and best performance over random splits of the data.
    \adasap{} performs similar to the competition.}
    \label{fig:all_appendix_time}
\end{figure}

\begin{figure}
    \centering
    \includegraphics[width=\linewidth]{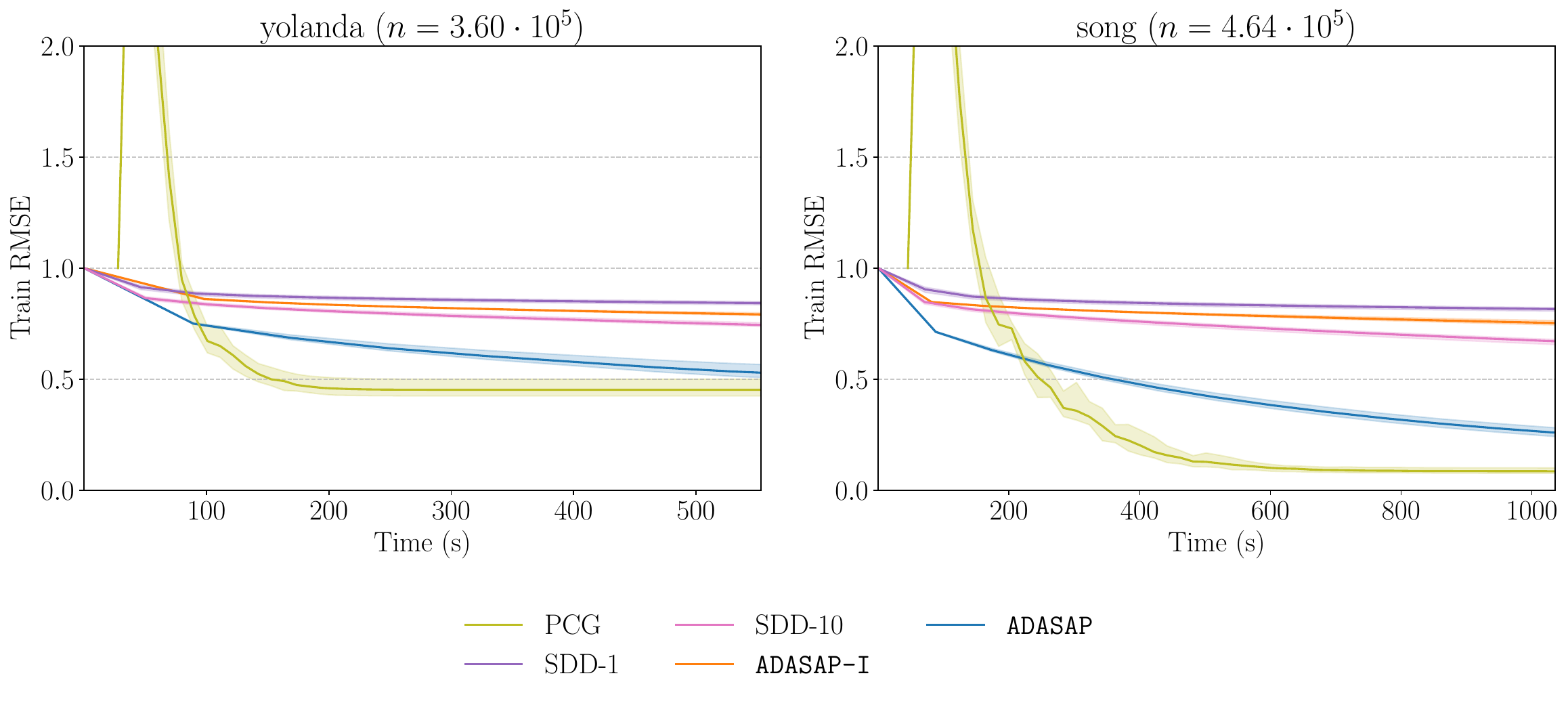}
    \caption{Performance of \adasap{} and competitors on train RMSE, as a function of time, for yolanda and song.
    The solid curve indicates mean performance over random splits of the data;
    the shaded regions indicate the range between the worst and best performance over random splits of the data.}
    \label{fig:all_appendix_time_train_rmse}
\end{figure}

% don't insert the checklist into the preprint
\ifarxivpreprint
\else
%    \newpage    
%    \input{sections/appendix/neurips_checklist}
\fi

%%%%%%%%%%%%%%%%%%%%%%%%%%%%%%%%%%%%%%%%%%%%%%%%%%%%%%%%%%%%

%%%% UNCOMMENT FOR CAMERA READY
%\newpage
%\input{sections/appendix/neurips_checklist}

\end{document}